\documentclass{article}

     \usepackage[preprint,nonatbib]{neurips_2022}

\usepackage[utf8]{inputenc} 
\usepackage[T1]{fontenc}    
\usepackage{hyperref}       
\usepackage{url}            
\usepackage{booktabs}       
\usepackage{amsfonts}       
\usepackage{nicefrac}       
\usepackage{microtype}      

\usepackage[dvipsnames]{xcolor}

\usepackage[numbers]{natbib}
\usepackage{amsmath}
\usepackage{amssymb}
\usepackage{amsthm}
\usepackage{caption}
\usepackage{subcaption}
\usepackage{graphicx}
\usepackage[symbol]{footmisc}
\usepackage{cancel}

\hypersetup{colorlinks,citecolor=blue,linkcolor=red,urlcolor=blue}

\DeclareFontFamily{OT1}{pzc}{}
\DeclareFontShape{OT1}{pzc}{m}{it}{<-> s * [1.10] pzcmi7t}{}
\DeclareMathAlphabet{\mathpzc}{OT1}{pzc}{m}{it}

\usepackage{tabularx}
\usepackage{booktabs} 

\newtheorem{theorem}{Theorem}
\newtheorem{corollary}{Corollary}
\newtheorem{proposition}{Proposition}
\newtheorem{lemma}[theorem]{Lemma}
\theoremstyle{definition}

\newcommand{\bE}{\boldsymbol{\mathrm{E}}}
\newcommand{\bX}{\boldsymbol{\mathrm{X}}}

\newcommand{\bx}{\boldsymbol{x}}
\newcommand{\by}{\boldsymbol{y}}

\newcommand{\bw}{\boldsymbol{\mathrm{w}}}

\newcommand{\bI}{\boldsymbol{\mathrm{I}}}

\newcommand{\bmu}{\boldsymbol{\mu}}

\newcommand{\bPhi}{\boldsymbol{\mathrm{\Phi}}}

\DeclareMathOperator*{\argmax}{arg\,max}

\newcommand{\look}[1]{{\textcolor{Red}{ #1 }}}

\title{Cold Posteriors through PAC-Bayes}

\author{
Konstantinos Pitas, Julyan Arbel\\ 
Univ. Grenoble Alpes, Inria, CNRS, Grenoble INP, LJK, 38000 Grenoble, France\\
\texttt{pitas.konstantinos@inria.fr,julyan.arbel@inria.fr}\\
}

\begin{document}

\maketitle

\begin{abstract}
We investigate the cold posterior effect through the lens of PAC-Bayes generalization bounds. We argue that in the non-asymptotic setting, when the number of training samples is (relatively) small, discussions of the cold posterior effect should take into account that approximate Bayesian inference does not readily provide guarantees of performance on out-of-sample data. Instead, out-of-sample error is better described through a generalization bound. In this context, we explore the connections of the ELBO objective from variational inference and the PAC-Bayes objectives. We note that, while the ELBO and PAC-Bayes objectives are similar, the latter objectives naturally contain a temperature parameter $\lambda$ which is not restricted to be $\lambda=1$. For both regression and classification tasks, in the case of isotropic Laplace approximations to the posterior,  we show how this PAC-Bayesian interpretation of the temperature parameter captures the cold posterior effect.
\end{abstract}

\section{Introduction}

\begin{figure*}[h!]
\centering
\begin{subfigure}{.25\textwidth}
  \centering
  \includegraphics[width=\textwidth]{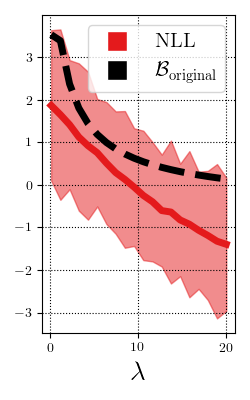}
  \caption{UCI, Abalone}
  \label{intro_fig:figure1}
\end{subfigure}%
\begin{subfigure}{.25\textwidth}
  \centering
  \includegraphics[width=\textwidth]{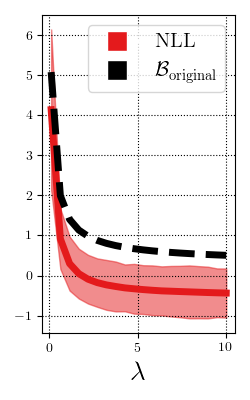}
  \caption{UCI, Diamonds}
  \label{intro_fig:figure2}
\end{subfigure}%
\begin{subfigure}{.25\textwidth}
  \centering
  \includegraphics[width=\textwidth]{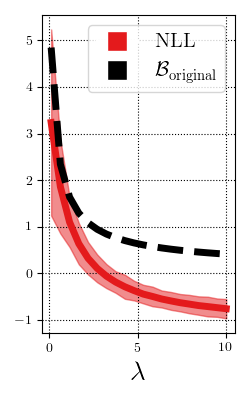}
  \caption{KC\_House}
  \label{intro_fig:figure3}
\end{subfigure}%
\begin{subfigure}{.25\textwidth}
  \centering
  \includegraphics[width=\textwidth]{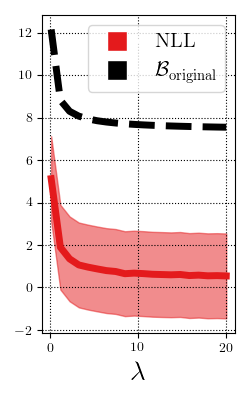}
  \caption{MNIST-10}
  \label{intro_fig:figure4}
\end{subfigure}%
\caption{The $\mathcal{B}_{\mathrm{original}}$ PAC-Bayes bound and the test negative log likelihood for different values of $\lambda$ (quantities on the y-axis are normalized). (a-c) are regression tasks on the UCI Abalone, UCI Diamonds and KC\_House datasets, while (d) is a classification task on the MNIST-10 dataset. $\mathcal{B}_{\mathrm{original}}$ PAC-Bayes bound closely tracks the test negative log-likelihood.}
  \label{intro_fig:figure_full}
\end{figure*}

In their influential paper, \citet{wenzel2020good} highlighted the observation that Bayesian neural networks typically exhibit better test time predictive performance if the posterior distribution is ``sharpened'' through tempering. Their work has been influential primary because it serves as a well documented example of the potential drawbacks of the Bayesian approach to deep learning. While other subfields of deep learning have seen rapid adoption, and have had impact on real world problems, Bayesian deep learning has, to date, seen relatively little practical use \cite{izmailov2021bayesian,lotfi2022bayesian,dusenberry2020efficient,wenzel2020good}. The ``cold posterior effect'', as the authors of \citet{wenzel2020good} named their observation, highlights an important mismatch between Bayesian theory and practice. As we increase the number of training samples, Bayesian theory tells us that we should be concentrating more and more on the true model parameters, in a frequentist sense. At any given moment, the posterior is our best guess at what the true model parameters are, without having to resort to heuristics. Since the original paper, a number of works \cite{noci2021disentangling,zeno2020cold,adlam2020cold,nabarro2021data,fortuin2021bayesian,aitchison2020statistical} have tried to explain the cold posterior effect, identify its origins, propose remedies and defend Bayesian deep leaning in the process. 

The experimental setups where the cold posterior effect arises have, however, been hard to pinpoint precisely. In \citet{noci2021disentangling} the authors conducted detailed experiments testing various hypotheses. The cold posterior effect was shown to arise from augmenting the data during optimization (data augmentation hypothesis), from selecting only the ``easiest'' data samples when constructing the dataset (data curation hypothesis), and from selecting a ``bad'' prior (prior misspecification hypothesis). In \citet{nabarro2021data}  the authors propose a principled log-likelihood that incorporates data augmentation, however they show that the cold-posterior persists. Data curation was first proposed as an explanation in \citet{aitchison2020statistical}, however the authors show that data curation can only explain a part of the cold posterior effect. Misspecified priors  have also been explored as a possible cause in several other works \cite{zeno2020cold,adlam2020cold,fortuin2021bayesian}. Again the results have been mixed. In smaller models, data dependent priors seem to decrease the cold posterior effect while in larger models the effect increases \cite{fortuin2021bayesian}. 

We propose that discussions of the cold posterior effect should take into account that in the \textit{non-asymptotic setting} (where the number of training data points is relatively small), Bayesian inference does not readily provide a guarantee for \textit{performance on out-of-sample data}. Existing theorems describe \emph{posterior contraction} \cite{ghosal2000convergence,blackwell1962merging}, however in practical settings, for a finite number of training steps and for finite training data, it is often difficult to \emph{precisely} characterise how much the posterior concentrates. Furthermore, theorems on posterior contraction are somewhat unsatisfying in the supervised classification setting, in which the cold posterior effect is usually discussed. Ideally, one would want a theoretical analysis that links the posterior distribution to the \emph{test} error directly. 

Here, we investigate PAC-Bayes generalization bounds \cite{mcallester1999some,catoni2007pac,alquier2016properties,dziugaite2017computing} as the model that governs performance on out-of-sample data. PAC-Bayes bounds describe the performance on out-of-sample data, through an application of the convex duality relation between measurable functions and probability measures. The convex duality relation naturally gives rise to the log-Laplace transform of a special random variable \cite{catoni2007pac}. Importantly the log-Laplace transform has a temperature parameter $\lambda$ which is not constrained to be $\lambda=1$. We investigate the relationship of this temperature parameter to cold posteriors. Our contributions are the following: 1) We prove a PAC-Bayes bound for linearized deep neural networks, that has a simple analytical form with respect to $\lambda$. This provides useful intuition for potential causes of the cold posterior effect. 2) For isotropic Laplace approximations to the posterior, for both regression and classification tasks, we show that a related PAC-Bayes bound correlates with performance on out-of-sample data. Our bounds are \emph{oracle} bounds, in that some quantities are typically unknown in real settings. We also rely on Monte Carlo sampling to estimate some quantities, and do not attempt to bound the error of these estimates. Even with these caveats we believe that our analysis highlights an important aspect of the cold posterior effect that has up to now been overlooked.

\section{The cold posterior effect in the misspecified and non-asymptotic setting}

We denote the learning sample $(X,Y)=\{(\bx_i,y_i)\}^n_{i=1}\in(\mathcal{X}\times\mathcal{Y})^n$, that contains $n$ input-output pairs. Observations $(X,Y)$ are assumed to be sampled randomly from a distribution $\mathcal{D}$. Thus, we denote $(X,Y)\sim\mathcal{D}^n$ the i.i.d observation of $n$ elements. We consider loss functions $\ell:\mathcal{F}\times\mathcal{X}\times\mathcal{Y}\rightarrow\mathbb{R}$, where $\mathcal{F}$ is a set of predictors $f:\mathcal{X}\rightarrow\mathcal{Y}$. We also denote the risk  $\mathcal{L}^{\ell}_{\mathcal{D}}(f)=\bE_{(\bx,y)\sim\mathcal{D}}\ell(f,\bx,y)$ and the empirical risk $\hat{\mathcal{L}}^{\ell}_{X,Y}(f)=(1/n)\sum_i\ell(f,\bx_i,y_i)$. We encounter cases where we make predictions using the posterior predictive distribution $\bE_{f\sim\hat{\rho}}[p(y|\bx,f)]$, with some abuse of notation we write the corresponding risk and empirical risk terms as $\mathcal{L}^{\ell}_{\mathcal{D}}(\bE_{f\sim\hat{\rho}})$ and $\hat{\mathcal{L}}^{\ell}_{X,Y}(\bE_{f\sim\hat{\rho}})$ correspondingly.

We will use two loss functions, the non-differentiable zero-one loss $\ell_{01}(f,\bx,y)=\mathbb{I}(\argmax(f(\bx))=y)$, and the negative log-likelihood, which is a commonly used differentiable surrogate $\ell_{\text{nll}}(f,\bx,y)=-\log(p(y|\bx,f))$, where we assume that the outputs of $f$ are normalized to form a probability distribution. Given the above, denoting the prior by $\pi$, the Evidence Lower Bound (ELBO) has the following form
\begin{equation}\label{generalized_obj}
-\bE_{f\sim\hat{\rho}}\hat{\mathcal{L}}^{\ell_{\mathrm{nll}}}_{X,Y}(f)-\frac{1}{\lambda n}\mathrm{KL}(\hat{\rho}\Vert \pi),
\end{equation}
where $\lambda=1$. Note that our temperature parameter $\lambda$ is the \emph{inverse} of the one typically used in cold posterior papers. In this form $\lambda$ has a clearer interpretation as the temperature of a log-Laplace transform. There is a slight ambiguity between tempered and cold posteriors, as argued in \citet{aitchison2020statistical} for Gaussian priors and posteriors the two objectives are equivalent. Overall our setup is equivalent to the one in \citet{wenzel2020good}. One typically models the posterior and prior distributions over weights using a parametric distribution (commonly a Gaussian) and optimizes the ELBO, using the reparametrization trick, to find the posterior distribution \cite{blundell2015weight,khan2018fast,mishkin2018slang,ashukha2019pitfalls,wenzel2020good}. The cold posterior is the following observation: 
\begin{quote}
\emph{Even though the ELBO has the form \ref{generalized_obj} with $\lambda = 1$, practitioners have found that much larger values $\lambda\gg1$ typically result in worse test time performance, for example a higher test misclassification rate and higher test negative log-likelihood.}
\end{quote} 
%

The starting point of our discussion will be thus to define the quantity that we care about in the context of Bayesian deep neural networks and cold posterior analyses. Concretely, in the setting of supervised prediction, what we often try to minimize is
\begin{equation}
    \label{eq:eval-metric-classif}
    \mathrm{KL}(p_{\mathcal{D}}(y|\bx)\Vert \bE_{f\sim\hat{\rho}}[p(y|\bx,f)])=\bE_{\bx,y\sim\mathcal{D}}\left[\ln\frac{p_{\mathcal{D}}(y|\bx)}{\bE_{f\sim\hat{\rho}}[p(y|\bx,f)]}\right],
\end{equation} 
the conditional relative entropy \citep{cover1999elements} between the true conditional distribution $p_{\mathcal{D}}(y|\bx)$ and $\bE_{f\sim\hat{\rho}}[p(y|\bx,f)]$ the posterior predictive distribution. For example, this is implicitly the quantity that we minimize when optimizing classifiers using the cross-entropy loss \cite{masegosa2019learning,morningstar2022pacm}. It determines how accurately
we can predict the future, it is often what governs how
much money our model will make or how many lives
it will save. It is also on this and similar predictive metrics that the cold posterior appears. In the following we will outline the relationship between the ELBO, PAC-Bayes and \ref{eq:eval-metric-classif}.  

\subsection{ELBO}

We assume a training sample $(X,Y)\sim\mathcal{D}^n$ as before, denote $p(\bw|X,Y)$ the true posterior probability over predictors $f$ parameterized by $\bw$ (typically weights for neural networks), and $\pi$ and $\hat{\rho}$ respectively the prior and variational posterior distributions as before. The ELBO results from the following calculations
\begin{equation*}
\begin{split}
\mathrm{KL}(\hat{\rho}(\bw)\Vert p(\bw|X,Y))&=\int \hat{\rho}(\bw) \ln \frac{\hat{\rho}(\bw)}{p(\bw|X,Y)}d \bw=\int \hat{\rho}(\bw) \ln \frac{\hat{\rho}(\bw)p(Y|X)}{\pi(\bw) p(Y|X,\bw)}d \bw\\
&=\int \hat{\rho}(\bw) 
\left[
- \ln p(Y|X,\bw) 
+ \ln \frac{\hat{\rho}(\bw)}{\pi(\bw)}
+  \ln p(Y|X) 
\right]
d \bw\\
&=-n\textcolor{orange}{\underbrace{\left(-\bE_{f\sim\hat{\rho}}\hat{\mathcal{L}}_{X,Y}^{\ell_{\mathrm{nll}}}(f)-\frac{1}{n}\mathrm{KL}(\hat{\rho}\Vert \pi)\right)}_{\mathrm{ELBO}}}+\ln p(Y|X).
\end{split}
\end{equation*}
Thus, maximizing the ELBO can be seen as minimizing the KL divergence between the true posterior and the variational posterior over the weights $\mathrm{KL}(\hat{\rho}(\bw)\Vert p(\bw|X,Y))$. The true posterior distribution $p(\bw|X,Y)$ gives more probability mass to predictors which are more likely given the training data, however these predictors do not necessarily minimize $\mathrm{KL}(p_{\mathcal{D}}(y|\bx)\Vert \bE_{f\sim\hat{\rho}}[p(y|\bx,f)])$, the evaluation metric of choice \eqref{eq:eval-metric-classif} for supervised classification. 

It is well known that Bayesian inference is strongly consistent under very broad conditions \cite{ghosal2000convergence}.
For example let the set of predictors $\mathcal{F}$ be countable and suppose that the data distribution $\mathcal{D}$ is such that, for some $f^*\in\mathcal{F}$ we have that $p(y|\bx,f^*)$ is equal to $p_{\mathcal{D}}(y|\bx)$ the true conditional distribution of $y$ given $\bx$. Then the Blackwell--Dubins consistency theorem \cite{blackwell1962merging} implies that with $\mathcal{D}$-probability 1, the Bayesian posterior concentrates on $f^*$. 
In supervised classification methods such as SVMs, the number of parameters is typically much smaller than the number of samples. In this situation, it is reasonable to assume that we are operating in the regime where $n \rightarrow \infty$ and that the posterior quickly concentrates on the true set of parameters. In such cases, a more detailed analysis, such as a PAC-Bayesian one, is unnecessary as the posterior is akin to a Dirac delta mass at the true parameters. However neural networks do not operate in this regime. In particular they are heavily overparametrized such that Bayesian model averaging always occurs empirically. In such cases,  it is often difficult to \emph{precisely} characterise how much the posterior concentrates. Furthermore, ideally, one would want a theoretical analysis that links the posterior distribution to the \emph{test} error directly. 

There is a more subtle cause that undermines consistency theorems that have worked well in the past, specifically \emph{model misspecification}. As shown in \citet{grunwald2007suboptimal}, for the case of supervised classification, there are two cases of misspecification where the Bayesian posterior does not concentrate to the optimal distribution with respect to the true risk even with infinite training data $n \rightarrow \infty$. 1) Assuming homoskedastic noise in the likelihood, when some data samples are corrupted with higher level noise than others. 2) The set of all predictors $\mathcal{F}$ does not include the true predictor $f^*$.
Both types of misspecification probably occur for deep neural networks. For example, accounting for heteroskedastic noise \cite{collier2021correlated} improves performance on some classification benchmarks. 
And the existence of multiple minima is a clue that no single best parametrization exists \cite{masegosa2019learning}. 

Operating in the regime where $n$ is (comparatively) small and where $f^* \notin \mathcal{F}$ makes it important to derive a more precise certificate of generalization through a generalization bound, which directly bounds the true risk. In the following we focus on analyzing a PAC-Bayes bound in order to obtain insights into when the cold posterior effect occurs.

\subsection{PAC-Bayes} 

We first look at the following bound, that we name the original bound and denote it by $\mathcal{B}_{\mathrm{original}}$.
\begin{theorem}[$\mathcal{B}_{\mathrm{original}}$, \cite{alquier2016properties}]\label{th:alquier}
Given a distribution $\mathcal{D}$ over $\mathcal{X}\times\mathcal{Y}$, a hypothesis set $\mathcal{F}$, a loss function $\ell:\mathcal{F}\times\mathcal{X}\times\mathcal{Y}\rightarrow\mathbb{R}$, a prior distribution $\pi$ over $\mathcal{F}$, real numbers $\delta \in (0,1]$ and $\lambda>0$, with probability at least $1-\delta$ over the choice $(X,Y)\sim\mathcal{D}^n$, we have for all $\hat{\rho}$ on $\mathcal{F}$
\begin{equation*}
\bE_{f\sim\hat{\rho}}\mathcal{L}_{\mathcal{D}}^{\ell}(f)\leq \bE_{f\sim\hat{\rho}}\hat{\mathcal{L}}_{X,Y}^{\ell}(f)+\frac{1}{\lambda n}\left[\mathrm{KL}(\hat{\rho}\Vert \pi)+\ln\frac{1}{\delta}   + \Psi_{\ell,\pi,\mathcal{D}}(\lambda,n) \right]
\end{equation*}
\begin{equation*}
\textit{where} \; \Psi_{\ell,\pi,\mathcal{D}}(\lambda,n)=\ln \bE_{f\sim\pi}\bE_{X',Y'\sim\mathcal{D}^n}\exp \left[\lambda n\left(\mathcal{L}_{\mathcal{D}}^{\ell}(f)- \hat{\mathcal{L}}_{X',Y'}^{\ell}(f)\right)  \right].
\end{equation*}
\end{theorem}
There are three different terms in the above bound:
\begin{equation*}
\textcolor{red}{\underbrace{\bE_{f\sim\hat{\rho}}\hat{\mathcal{L}}_{X,Y}^{\ell}(f)}_{\text{Empirical Risk}}}+\textcolor{blue}{\underbrace{\frac{1}{\lambda n}\left[\mathrm{KL}(\hat{\rho}\Vert \pi)+\ln\frac{1}{\delta} \right]}_{\text{KL}}}   + \textcolor{Green}{\underbrace{\frac{1}{\lambda n}\Psi_{\ell,\pi,\mathcal{D}}(\lambda,n)}_{\text{Moment}}}.
\end{equation*} 
The empirical risk term is the empirical mean of the loss of the classifier over all training samples. The KL term is the complexity of the model, which in this case is measured as the KL-divergence between the posterior and prior distributions. The Moment term, this is the log-Laplace transform for a reversal of the temperature, we will keep the name ``Moment'' in the following.

Using a PAC-Bayes bound together with Jensen's inequality, one can bound \eqref{eq:eval-metric-classif} directly as follows
\begin{equation*}
\begin{split}
&\mathrm{KL}(p_{\mathcal{D}}(y|\bx)\Vert \bE_{f\sim\hat{\rho}}[p(y|\bx,f)])=\bE_{\bx,y\sim\mathcal{D}}\left[\ln\frac{p_{\mathcal{D}}(y|\bx)}{\bE_{f\sim\hat{\rho}}[p(y|\bx,f)]}\right]\\
&=\bE_{\bx,y\sim\mathcal{D}}[-\ln\bE_{f\sim\hat{\rho}}[p(y|\bx,f)]]+\bE_{\bx,y\sim\mathcal{D}}[\ln p_{\mathcal{D}}(y|\bx)]\\
&\leq\bE_{\bx,y\sim\mathcal{D}}[\bE_{f\sim\hat{\rho}}[-\ln p(y|\bx,f)]]+\bE_{\bx,y\sim\mathcal{D}}[\ln p_{\mathcal{D}}(y|\bx)]\\
&\leq\textcolor{orange}{\underbrace{\bE_{f\sim\hat{\rho}}\hat{\mathcal{L}}_{X,Y}^{\ell_{\mathrm{nll}}}(f)+\frac{1}{\lambda n}\left[\mathrm{KL}(\hat{\rho}\Vert \pi)+\ln\frac{1}{\delta}+ \Psi_{\ell_{\mathrm{nll}},\pi,\mathcal{D}}(\lambda,n) \right]}_{\mathrm{PAC}\text{-}\mathrm{Bayes} }}    +\bE_{\bx,y\sim\mathcal{D}}[\ln p_{\mathcal{D}}(y|\bx)].
\end{split}
\end{equation*}
%
The last line holds under the conditions of Theorem~\ref{th:alquier} and in particular with probability at least $1-\delta$ over the choice $(X,Y)\sim\mathcal{D}^n$. Notice here the presence of the temperature parameter $\lambda\geq0$, which needs not be $\lambda=1$. 
\begin{quote}
    \emph{In particular it is easy to see that maximizing the ELBO is equivalent to minimizing a PAC-Bayes bound for $\lambda=1$, which might not necessarily be optimal for a finite sample size. More specifically even for exact inference, where $\bE_{\bw\sim\hat{\rho}}[p(y|\bx,\bw)]|_{\hat{\rho}=p(\bw|X,Y)}=p(y|\bx,X,Y)$, the Bayesian posterior predictive distribution does not necessarily minimize $\mathrm{KL}(p_{\mathcal{D}}(y|\bx)\Vert \bE_{f\sim\hat{\rho}}[p(y|\bx,f)])$.}
\end{quote}

\subsection{Safe-Bayes and other relevant work}
We are not the first to discuss the relationship of Bayesian inference and PAC-Bayes, nor the connection to tempered posteriors, however we are the first investigate the relationship with the cold posterior effect in the context of deep learning. In \citet{germain2016pac} the authors where the first to find connections between PAC-Bayes and Bayesian inference. However they only investigate the case where $\lambda=1$. After identifying two sources of misspecification the authors in \citet{grunwald2007suboptimal} proposed a solution, through an approach which they named Safe-Bayes \cite{grunwald2012safe,grunwald2017inconsistency}. Safe-Bayes corresponds to finding a temperature parameter $\lambda$ for a generalized (tempered) posterior distribution with $\lambda$ possibly different than 1. The optimal value of $\lambda$ is found by taking a sequential view of Bayesian inference and minimizing a prequential risk, the risk of each new data sample given the previous ones. This results in a PAC-Bayes bound on the true risk, and is reminiscent of recent works in Bayesian inference and model selection such as \citet{lyle2020bayesian,ru2020speedy}. The analysis of \citet{grunwald2012safe,grunwald2017inconsistency} is restricted to the case where $\lambda<1$. By contrast we provide an analytical expression of the bound on true risk, given $\lambda$, and also numerically investigate the case of $\lambda>1$. Our analysis thus provides intuition regarding which parameters (for example the curvature $h$) might result in cold posteriors. 

\section{The effect of the temperature parameter $\lambda$ on the PAC-Bayes bound}
PAC-Bayes objectives are typically difficult to analyze theoretically. In the following we make a number of simplifying assumptions, thus making deep neural networks amenable to study.
We exploit the idea of a recent line of works \cite{zancato2020predicting,maddox2021fast,jacot2018neural,khan2019approximate} that have considered linearizations of deep neural networks, at some estimate $\bw_{\hat{\rho}}$, such that
\begin{equation}
    \label{eq:lin-NN}
    f_{\mathrm{lin}}(\bx;\bw)=f(\bx;\bw_{\hat{\rho}})+\nabla_{\bw}f(\bx;\bw_{\hat{\rho}})^{\top}(\bw-\bw_{\hat{\rho}})
\end{equation}
to derive theoretical results. Our approach is somewhat connected to the NTK \cite{jacot2018neural}, however it is much closer to \citet{zancato2020predicting,maddox2021fast,khan2019approximate} as we make no assumptions about infinite width. For appropriate modelling choices, we aim at deriving a bound for this linearized model. 

We adopt the linear form~\eqref{eq:lin-NN} together with the Gaussian likelihood with $\sigma=1$, yielding $\ell_{\mathrm{nll}}(\bw,\bx,y)=\frac{1}{2}\ln(2\pi)+\frac{1}{2}(y-f(\bx;\bw_{\hat{\rho}})-\nabla_{\bw}f(\bx;\bw_{\hat{\rho}})^{\top}(\bw-\bw_{\hat{\rho}}))^2$. We also make the following modeling choices
\begin{itemize}
\item Prior over weights: $\bw \sim \mathcal{N}(\bw_{\pi},\sigma_{\pi}^2\bI)$.
\item Gradient as Gaussian mixtures: $\nabla_{\bw}f(\bx;\bw_{\hat{\rho}})\sim\sum_{i=1}^k\phi_i\mathcal{N}(\bmu_i,\sigma_{\bx i}^2\bI)$; note that this assumption should be somewhat realistic for pretrained neural networks, in that multiple works have shown that gradients with respect to the training, at $\bw_{\hat{\rho}}$, set are clusterable \cite{zancato2020predicting}. 
\item Labeling function: $y = f(\bx;\bw_{\hat{\rho}})+\nabla_{\bw}f(\bx;\bw_{\hat{\rho}})^{\top}(\bw_*-\bw_{\hat{\rho}})+\epsilon$, where $\epsilon\sim\mathcal{N}(0,\sigma_{\epsilon}^2)$. 
\end{itemize}
Thus $y|\bx\sim\mathcal{N}(f(\bx;\bw_{\hat{\rho}})+\nabla_{\bw}f(\bx;\bw_{\hat{\rho}})^{\top}(\bw_*-\bw_{\hat{\rho}}),\sigma_{\epsilon}^2)$. The assumption that $\bw_*$ is close to $\bw_{\hat{\rho}}$ is quite strong, and we furthermore argued in the previous sections that no single $\bw$ is truly ``correct''. However we note that for fine-tuning tasks linearized neural networks work remarkably well \cite{maddox2021fast,deshpande2021linearized}. It is therefore at least somewhat reasonable to assume the above oracle labelling function, in that for deep learning architectures good $\bw$ that fit many datasets can be found close to $\bw_{\hat{\rho}}$ in practical settings. In any case as we will see later we will only look for some rough intuition from this analysis, and will resort to \ref{th:alquier} for practical settings.

We also assume that we have a deterministic estimate of the posterior weights $\bw_{\hat{\rho}}$ \emph{which we keep fixed}, and we model the posterior as $\hat{\rho}=\mathcal{N}(\bw_{\hat{\rho}},\sigma^2_{\hat{\rho}}(\lambda)\bI)$. Therefore estimating the posterior corresponds to estimating the variance $\sigma^2_{\hat{\rho}}(\lambda)$. This setting has been widely explored before in the literature as it coincides with the Laplace approximation to the posterior.
\begin{proposition}[$\mathcal{B}_{\mathrm{approximate}}$]\label{dnn_approximate}
With the above modeling choices, and given a distribution $\mathcal{D}$ over $\mathcal{X}\times\mathcal{Y}$, real numbers $\delta \in (0,1]$ and $\lambda \in (0,\frac{1}{c})$ with $c = 2 n\sigma^2_{\bx}\sigma^2_{\pi}$,
 with probability at least $1-\delta$ over the choice $(X,Y)\sim\mathcal{D}^n$, we have
\begin{equation*}
\begin{split}
&\bE_{\bw\sim\hat{\rho}}\mathcal{L}_{\mathcal{D}}^{\ell_{\mathrm{nll}}}(\bw)\\
& \qquad \leq \textcolor{red}{\underbrace{\frac{\Vert \by-f(\bX;\bw_{\hat{\rho}})\Vert ^2_2}{2n}+\left(\frac{\lambda h}{d}+\frac{1}{\sigma_{\pi}^2}\right)^{-1} \frac{h}{2n}+\frac{1}{2}\ln(2\pi) }_{\normalfont{\text{Empirical Risk}}}} + \textcolor{Green}{\underbrace{\frac{\sigma_{\bx}^2(\sigma_{\pi}^2d+\Vert \bw_*\Vert _2^2)}{1-2\lambda n \sigma_{\bx}^2\sigma_{\pi}^2}+\sigma_{\epsilon}^2}_{\normalfont{\text{Moment}}}}+\\
&\qquad\textcolor{blue}{\underbrace{\frac{1}{\lambda n} \left[\frac{1}{2}\left(\frac{d}{\sigma_{\pi}^2}  \frac{1}{\frac{\lambda h}{d}+\frac{1}{\sigma_{\pi}^2}} + \frac{1}{\sigma^2_{\pi}}\Vert \bw_{\hat{\rho}}-\bw_{\pi}\Vert ^2_2 -d-d \ln\frac{1}{\frac{\lambda h}{d}+\frac{1}{\sigma_{\pi}^2}}+d \ln\sigma_{\pi}^2 \right) +\ln\frac{1}{\delta}\right]}_{\normalfont{\text{KL}}}}
\end{split}
\end{equation*}
where $h = \sum_{i}\sum_{j}(\nabla_{\bw}f(\bx_i;\bw_{\hat{\rho}})_j)^2$ is the curvature parameter, and $\sigma_{\bx}^2=\sum_{j=1}^k\phi_j\sigma^2_{\bx j}$ is the posterior gradient variance.
\end{proposition}
\begin{proof}{(Sketch)}
We first develop all the terms in the PAC-Bayes bound based on our modelling choices. We start with the empirical risk term
\begin{equation*}
\begin{split}
&2n\bE_{\bw\sim\hat{\rho}}\hat{\mathcal{L}}_{X,Y}^{\ell_{\mathrm{nll}}}(\bw)  =\Vert\by-f(\bX;\bw_{\hat{\rho}})\Vert^2_2+ \sigma^2_{\hat{\rho}}(\lambda)h+n\ln(2\pi),\\
\end{split}
\end{equation*}
%
%
where we set $h = \sum_{i}\sum_{j}(\nabla_{\bw}f(\bx_i;\bw_{\hat{\rho}})_j)^2$. This factor $h$ is the only one multiplied with the posterior variance $\sigma^2_{\hat{\rho}}(\lambda)$. It can be interpreted as measuring the \emph{curvature} of the loss landscape. 

We continue with the KL term. For our modelling choice of Gaussian prior $\pi=\mathcal{N}(\bw_{\pi},\sigma^2_{\pi}\bI)$ and posterior $\hat{\rho}=\mathcal{N}(\bw_{\hat{\rho}},\sigma^2_{\hat{\rho}}(\lambda)\bI)$, the KL has the following analytical expression
\begin{multline*}
\mathrm{KL}(\mathcal{N}(\bw_{\hat{\rho}},\sigma^2_{\hat{\rho}}(\lambda)\bI)\Vert \mathcal{N}(\bw_{\pi},\sigma^2_{\pi}\bI)) +\ln\frac{1}{\delta}\\ 
=\frac{1}{2}\left(d\frac{\sigma^2_{\hat{\rho}}(\lambda)}{\sigma_{\pi}^2}   + \frac{1}{\sigma^2_{\pi}}\Vert \bw_{\hat{\rho}}-\bw_{\pi}\Vert ^2 -d-d \ln\sigma^2_{\hat{\rho}}(\lambda)+d \ln\sigma_{\pi}^2 \right) +\ln\frac{1}{\delta}.
\end{multline*}

Finally we develop the moment term. We find the following upper bound
\begin{equation*}
\begin{split}
\frac{1}{\lambda}\Psi_{\ell,\pi,\mathcal{D}}(\lambda,n)  
& = \ln \bE_{f\sim\pi}\bE_{X',Y'\sim\mathcal{D}^n}\exp \left[\lambda n\left(\mathcal{L}_{\mathcal{D}}^{\ell_{\mathrm{nll}}}(f)- \hat{\mathcal{L}}_{X',Y'}^{\ell_{\mathrm{nll}}}(f)\right)  \right]\\ 
& \leq \ln \bE_{f\sim\pi}\exp \left[\lambda n\left(\mathcal{L}_{\mathcal{D}}^{\ell_{\mathrm{nll}}}(f)\right)  \right]= \frac{\sigma^2_{\bx}(\sigma^2_{\pi}d+\Vert \bw_{*}\Vert ^2)}{(1-\lambda c)}+\sigma_{\epsilon}^2,\\
\end{split}
\end{equation*}
where $c = 2 n\sigma^2_{\bx}\sigma^2_{\pi}$ and $\lambda \in (0,\frac{1}{c})$. In bounding the moment term we first recognize that the random variable consists in the difference $\mathcal{L}_{\mathcal{D}}^{\ell_{\mathrm{nll}}}(f)-\hat{\mathcal{L}}_{X',Y'}^{\ell_{\mathrm{nll}}}(f)$. We remove the term $\hat{\mathcal{L}}_{X',Y'}^{\ell_{\mathrm{nll}}}(f)$ resulting in an upper bound to the moment, and in this way we avoid also having to calculate the expectation $\bE_{X',Y'\sim\mathcal{D}^n}$. For samples from the prior $f\sim\pi$, we can then compute the remaining expectation because we have assumed that the labelling function is known and equal to $y = \phi(\bx)\bw^{*}+\epsilon$ where $\epsilon\sim\mathcal{N}(0,\sigma_{\epsilon}^2)$. The cost of finding an analytical expression is that this bound is very loose. While $\mathcal{L}_{\mathcal{D}}^{\ell_{\mathrm{nll}}}(f)- \hat{\mathcal{L}}_{X',Y'}^{\ell_{\mathrm{nll}}}(f)$ can be small $\mathcal{L}_{\mathcal{D}}^{\ell_{\mathrm{nll}}}(f)$ by itself can be very large. We are also now saddled with the constraint $\lambda \in (0,\frac{1}{c})$ where $c = 2 \sigma^2_{\bx}\sigma^2_{\pi}$. We will see that this is also pessimistic.


We now minimize with respect to $\sigma^2_{\hat{\rho}}(\lambda)$ the following objective which is equivalent to the one of Equation (4) in \citet{wenzel2020good}
\begin{equation*}
\min_{\sigma^2_{\hat{\rho}}(\lambda)} \bE_{\bw\sim\hat{\rho}}\hat{\mathcal{L}}_{X,Y}^{\ell_{\mathrm{nll}}}(\bw) + \frac{1}{\lambda n}\left[\mathrm{KL}(\mathcal{N}(\bw_{\hat{\rho}},\sigma^2_{\hat{\rho}}(\lambda)\bI)\Vert \mathcal{N}(\bw_{\pi},\sigma^2_{\pi}\bI)) +\ln\frac{1}{\delta}\right].
\end{equation*}
For our particular modeling choices, we can then find the minimum in closed-form by setting the gradient with respect to $\sigma^2_{\hat{\rho}}(\lambda)$ to be zero. We get 
$
\sigma^2_{\hat{\rho}}(\lambda)=\frac{1}{\frac{\lambda h}{d}+\frac{1}{\sigma^2_{\pi}}}
$ 
where $h = \sum_{i=1}^d\sum_{j=1}^n\phi_i(\bx_j)^2$. We get our result by putting everything in the original bound.
\end{proof}

We now make a number of observations regarding Proposition~\ref{dnn_approximate}. The first is that we started our derivations from the linearized deep neural network and the negative log-likelihood with a Gaussian likelihood however the empirical risk term in Proposition~\ref{dnn_approximate} is exactly the Gauss--Newton approximation to the loss landscape for the full deep neural network. Here $h$ is the trace of the Hessian under the Gauss Newton approximation (without a scaling factor $n$). We refer to details about this interpretation of the empirical risk to the Appendix. 
We call the bound of Proposition~\ref{dnn_approximate} the $\mathcal{B}_{\mathrm{approximate}}$ bound. The r.h.s of the inequality in Proposition~\ref{dnn_approximate} is unwieldy, and provides little intuition. We simplify it by making the following parameter choices 
\begin{corollary}\label{simplification}
For $n=d=h=\sigma_{\pi}^2=\sigma_{\bx}^2=\Vert \bw_{*}\Vert _2^2=\Vert \bw_{\hat{\rho}}-\bw_{\pi}\Vert =\sigma_{\epsilon}^2=1$, $\Vert \by-f(\bX;\bw_{\hat{\rho}}) \Vert _2^2=0$, and ignoring additive constants, the dependence of Proposition~\ref{dnn_approximate} bound $\mathcal{B}_{\mathrm{approximate}}$ on the temperature parameter $\lambda$ is as follows, with probability at least $1-\delta$
\begin{equation}\label{simplification_eq}
\bE_{\bw\sim\hat{\rho}}\mathcal{L}_{\mathcal{D}}^{\ell_{\mathrm{nll}}}(\bw) \leq \textcolor{red}{\underbrace{\frac{1}{2}\frac{1}{\lambda+1}}_{\normalfont{\text{Empirical Risk} }} } 
+ \textcolor{Green}{\underbrace{\frac{2}{1-2\lambda}}_{\normalfont{\text{Moment} }}}+\textcolor{blue}{\underbrace{\frac{1}{\lambda} \left[\frac{1}{2}\left(\frac{1}{\lambda+1}-\ln\frac{1}{\lambda+1} \right) +\ln\frac{1}{\delta}\right]}_{\normalfont{\text{KL} }}}.
\end{equation}
\end{corollary}
\begin{figure*}[t!]
\centering
\begin{subfigure}{.33\textwidth}
  \centering
  \includegraphics[scale=0.4]{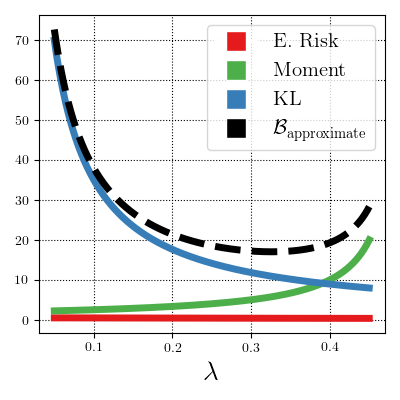}
  \caption{$\mathcal{B}_{\mathrm{approximate}}$ for Corollary~\ref{simplification}.}
  \label{theory_img:figure1}
\end{subfigure}%
\begin{subfigure}{.33\textwidth}
  \centering
  \includegraphics[scale=0.4]{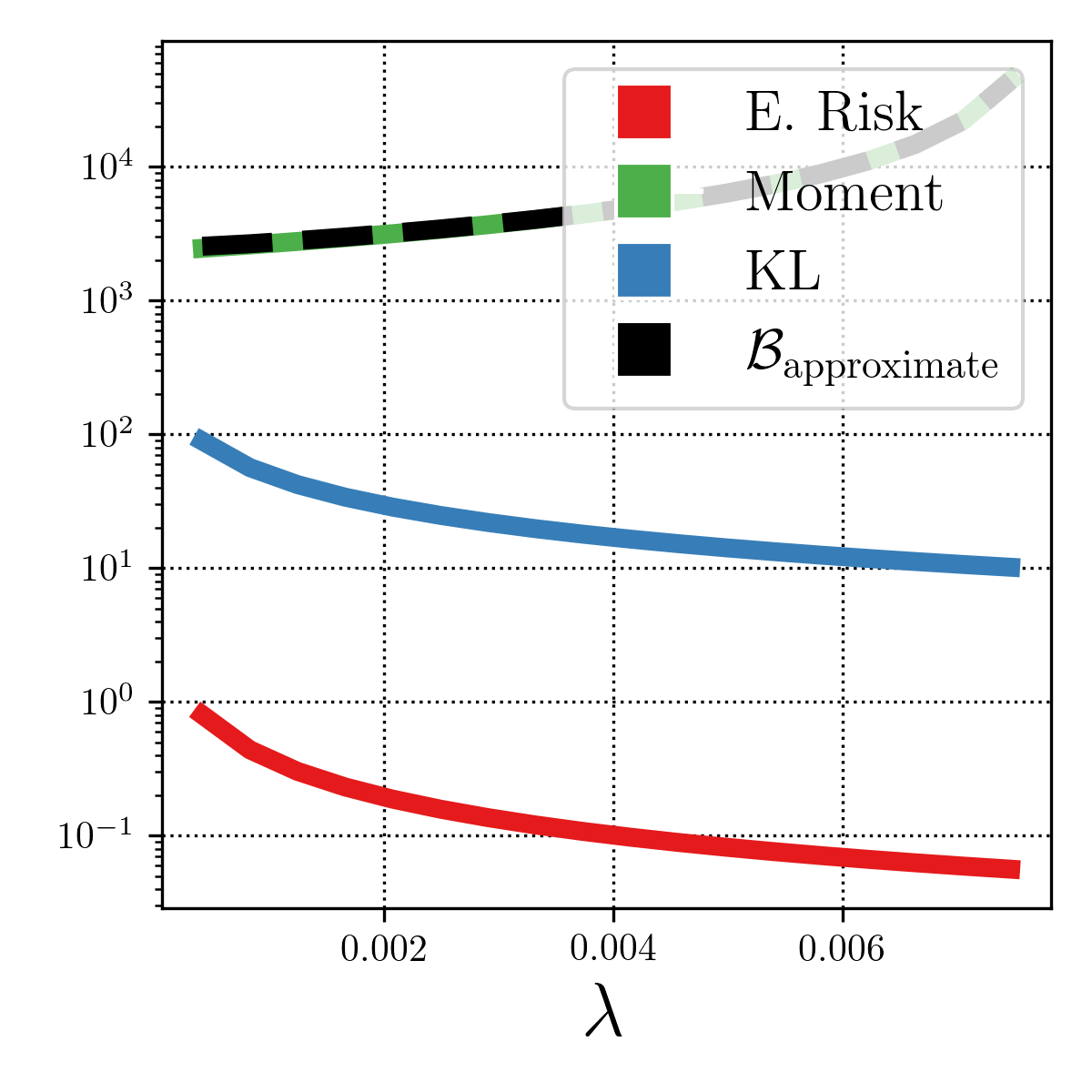}
  \caption{$\mathcal{B}_{\mathrm{approximate}}$ for KC\_House. }
  \label{theory_img:figure2}
\end{subfigure}%
\begin{subfigure}{.33\textwidth}
  \centering
  \includegraphics[scale=0.4]{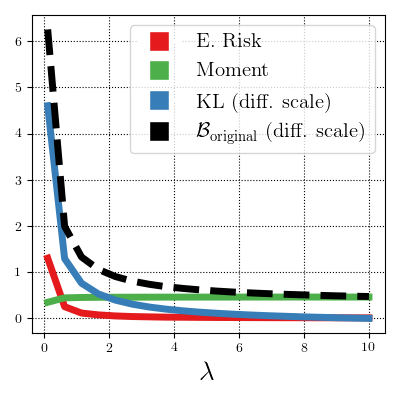}
  \caption{$\mathcal{B}_{\mathrm{original}}$ for KC\_House. }
  \label{theory_img:figure3}
\end{subfigure}%
\caption{$\mathcal{B}_{\mathrm{approximate}}$ as a function of the $\lambda$ parameter. We plot the empirical risk, moment and KL terms, as well as the $\mathcal{B}_{\mathrm{approximate}}$ bound values for different $\lambda$. In Figure \ref{theory_img:figure1} we plot the simplified case of Corollary~\ref{simplification}. We see that the empirical risk and KL terms decrease while the moment term decreases as $\lambda$ increases. The bound first decreases and then increases, and has a global minimum. In Figure \ref{theory_img:figure2} we plot the $\mathcal{B}_{\mathrm{approximate}}$ bound for a single MAP estimate of a neural network trained on the KC\_House dataset. The x-axis is contains only values $\lambda<1$, which limits our investigation of the cold posterior effect, also the Moment term explodes.  In Figure \ref{theory_img:figure3} we plot the $\mathcal{B}_{\mathrm{original}}$ bound for the same model and dataset. We can now investigate values of $\lambda>1$. The terms increase or decrease according to the intuition in Corollary~\ref{simplification}. However the KL term decreases faster than the moment increases and the Bound always \emph{decreases} as we increase $\lambda$.}
  \label{theory_img:figure_full}
\end{figure*}
From Equation \eqref{simplification_eq} it is easy to derive some intuition as to the effect of $\lambda$. We see that as $\lambda$ increases the empirical risk \emph{decreases} while the moment term \emph{increases}. For this particular modeling choice the KL term \emph{decreases} as we increase $\lambda$. We plot the $\mathcal{B}_{\mathrm{approximate}}$ bound for the parameter choices of Corollary~\ref{simplification} in Figure \ref{theory_img:figure1}. We see that as a function of $\lambda$ the bound first decreases and then increases, \emph{the bound is minimized for an intermediate value}. This analysis also gives us some interesting intuition regarding cold posteriors in general. Under the PAC-Bayesian modeling of the risk, cold posteriors are the result of a complex interaction between the various parameters of the bound. This might explain why pinpointing their cause is difficult in practice. 

In the next section we show that for real datasets and models, the bound terms often have the same behaviour as in Corollary~\ref{simplification}. However this is not the case for the overall bound, which has different shapes based on different hyperparameter values. In particular in some cases the bound as a function of $\lambda$ is sometimes not convex, while in others it does not have a global minimum, but decreases as $\lambda \rightarrow +\infty$.



\section{Experiments}
\begin{figure*}[t!]
\centering
\begin{subfigure}{.25\textwidth}
  \centering
  \includegraphics[width=\textwidth]{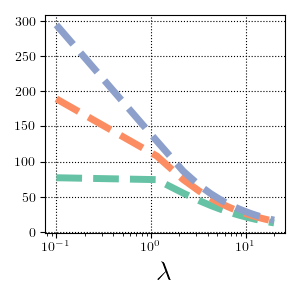}
\end{subfigure}%
\begin{subfigure}{.25\textwidth}
  \centering
  \includegraphics[width=\textwidth]{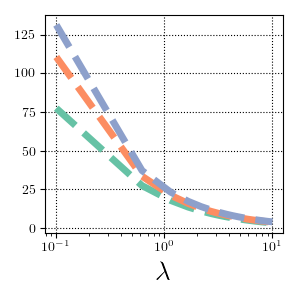}
\end{subfigure}%
\begin{subfigure}{.25\textwidth}
  \centering
  \includegraphics[width=\textwidth]{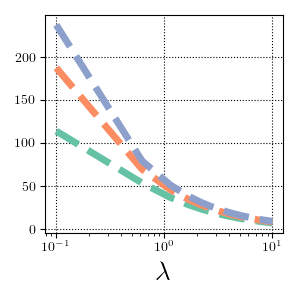}
\end{subfigure}%
\begin{subfigure}{.25\textwidth}
  \centering
  \includegraphics[width=\textwidth]{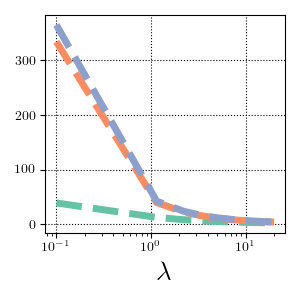}
\end{subfigure}

\begin{subfigure}{.25\textwidth}
  \centering
  \includegraphics[width=\textwidth]{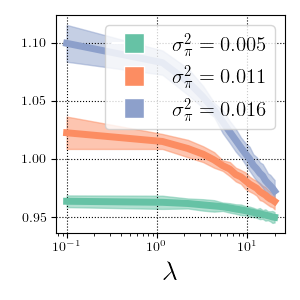}
  \caption{UCI, Abalone}
   \label{exp_fig:figure1}
\end{subfigure}%
\begin{subfigure}{.25\textwidth}
  \centering
  \includegraphics[width=\textwidth]{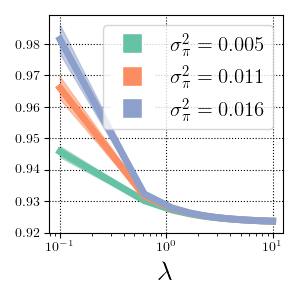}
  \caption{UCI, Diamonds}
   \label{exp_fig:figure2}
\end{subfigure}%
\begin{subfigure}{.25\textwidth}
  \centering
  \includegraphics[width=\textwidth]{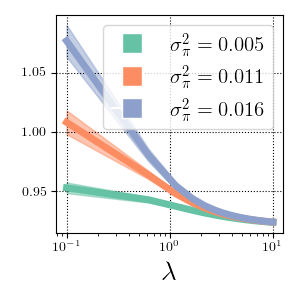}
  \caption{KC\_House}
   \label{exp_fig:figure3}
\end{subfigure}%
\begin{subfigure}{.25\textwidth}
  \centering
  \includegraphics[width=\textwidth]{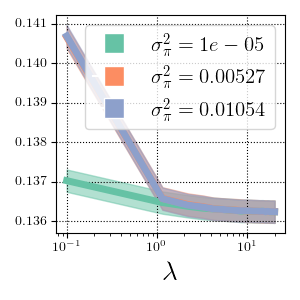}
  \caption{MNIST-10}
   \label{exp_fig:figure4}
\end{subfigure}%

  \caption{$\mathcal{B}_{\mathrm{original}}$ PAC-Bayes bound (top row) and negative log-likelihood (bottom row) for varying $\lambda$ (quantities on the y-axis are normalized). (a-c) are regression tasks on the UCI Abalone, UCI Diamonds and KC\_House datasets, while (d) is a classification task on the MNIST-10 dataset. The $\mathcal{B}_{\mathrm{original}}$ PAC-Bayes bound closely tracks the negative log-likelihood.}
  \label{exp_fig:figure_full}
\end{figure*}
We tested our theoretical results on three regression datasets and one classification dataset. The regression tasks are with the Abalone and Diamonds datasets from the UCI repository \cite{Dua:2019}, as well as with the popular ``House Sales in King County, USA'' (KC\_House) dataset from the Kaggle competition website \cite{kaggle}. We chose these datasets because they provide a large number of samples, compared to more common UCI datasets such as Boston and Wine. This larger number of samples makes it easier to compute oracle quantities. For the classification task we used the typical MNIST-10 dataset \cite{deng2012mnist}. When needing extra samples to approximate the complete distribution we used samples from the EMNIST \cite{cohen2017emnist} dataset.

In all experiments we split the dataset into five sets. Three of them are the typical prediction tasks sets: training set $Z_{\mathrm{train}}$, testing set $Z_{\mathrm{test}}$, and validation set $Z_{\mathrm{validation}}$. Our experimental setup requires two extra sets: the ``train-suffix'' set $Z_{\mathrm{trainsuffix}}$, as well as a large sample set called ``true'' set  $Z_{\mathrm{true}}$ that is used to approximate the complete distribution.

We will encounter several problems when trying to evaluate the $\mathcal{B}_{\mathrm{approximate}}$ bound in practice.  Firstly, as we mentioned, we expect the $\mathcal{B}_{\mathrm{approximate}}$ bound to be loose. We will therefore find it useful to evaluate the original bound, without any approximations. As we mentioned previously, we name this the $\mathcal{B}_{\mathrm{original}}$. We approximate \emph{both} the moment term \emph{and} the empirical risk term using Monte Carlo sampling with $f\sim\pi$ and $X',Y'\sim\mathcal{D}$, and $f\sim\hat{\rho}$. 
An additional problem when analyzing our classification model is that we derived the $\mathcal{B}_{\mathrm{approximate}}$ bound for model with a single output, while the MNIST-10 dataset has 10 classes and thus a 10-dimensional output. Furthermore classification models are not typically trained with the Gaussian likelihood, but with a softmax activation coupled with the categorical crossentropy loss. In the previous section we observed that for the linearized neural network and the Gaussian likelihood we arrived at a second order Taylor expansion of the loss landscape where the Hessian was approximated with the Gauss--Newton approximation. In the case of MNIST, we thus make this Taylor expansion directly on the full cross-entropy empirical loss $\hat{\mathcal{L}}^{\ell}_{X,Y}(f_{\bw})\approx\hat{\mathcal{L}}^{\ell}_{X,Y}(f_{\bw_{\hat{\rho}}})+\frac{\sigma_{\hat{\rho}}h}{2n}$ where $h = \left[\sum_{k}\sum_{i}\sum_{j}g(i,k)(\nabla_{\bw}f_k(\bx_i;\bw_{\hat{\rho}})_j)^2\right]$ is the trace of the Gauss--Newton approximation to the Hessian, for the categorical cross-entropy loss (without a scaling factor $n$) ($k$ denotes the output dimensions and $g(i,k)$ is a special function, see \cite{kunstner2019limitations} eq. 35 for details). We can now similarly to before calculate $\sigma^2_{\hat{\rho}}(\lambda)=\left(\frac{1}{\frac{\lambda h}{d}+\frac{1}{\sigma_{\pi}^2}}\right)$ for our Laplace approximation to the posterior.

PAC-Bayes bounds require correct control of the prior mean as the $\ell_2$ distance between prior and posterior means in the KL term is often the dominant term in the bound. To control this distance we follow a variation of the approach in \citet{dziugaite2021role} to constructing our classifiers. We first use the $Z_{\mathrm{train}}$ to find a prior mean $\bw_{\pi}$. We then set the posterior mean equal to the prior mean $\bw_{\hat{\rho}}=\bw_{\pi}$ but evaluate the r.h.s of the bounds on the $Z_{\mathrm{trainsuffix}}$. Note that in this way $\Vert \bw_{\hat{\rho}}-\bw_{\pi}\Vert _2^2=0$, while the bound is still valid since the prior is independent from the evaluation set $X,Y=Z_{\mathrm{trainsuffix}}$. The cost is that we deviate from normal practice. We have in essence constructed a Laplace approximation where the mean was learned using $Z_{\mathrm{train}}$ while the posterior variance was learned using $Z_{\mathrm{trainsuffix}}$. 

For the UCI and KC\_House experiments we use fully connected networks with 2 hidden layers with 100 dimensions, followed by the ReLU activation function, and a final Softmax activation. For the MNIST-10 dataset we use the standard LeNet architecture \cite{Lecun726791}. More details on the experimental setup can be found in the Appendix. 

\subsection{UCI +KC\_House experiments}
We first test the applicability of the $\mathcal{B}_{\mathrm{approximate}}$ bound in practice. For this we use the KC\_House dataset, although the results for the rest of the datasets are similar. We train the network on $Z_{\mathrm{train}}$ using SGD with stepsize $\eta=10^{-3}$ for 10 epochs. We then estimate the bound using $Z_{\mathrm{trainsuffix}}$. We give details on how all bound parameters are estimated in the Appendix. We plot the results in Figure \ref{theory_img:figure2}. While the terms of the bound move in the way we expected from Corollary~\ref{simplification} We see that the bound is significantly off scale on the y axis. Specifically it is constrained to be $\lambda\ll1$ which is far from the regime that we want to explore. In Figure \ref{theory_img:figure3} we thus plot for the same dataset the $\mathcal{B}_{\mathrm{original}}$ bound. We see that some of the problems of the $\mathcal{B}_{\mathrm{approximate}}$ bound have been fixed. Specifically we can now estimate bound values for $\lambda>1$. However surprisingly the bound \emph{always decreases} as we increase $\lambda$. This is because the moment term increases at a much slower rate than the KL term decreases. We thus use the $\mathcal{B}_{\mathrm{original}}$ bound in the rest of our experiments. 

In Figures \ref{exp_fig:figure_full} we plot the results for all the regression datasets. As before we train the networks on $Z_{\mathrm{train}}$ using SGD with stepsize $\eta=10^{-3}$ for 10 epochs. We then estimate the bound using $Z_{\mathrm{trainsuffix}}$. We also plot the \emph{test} NLL of the posterior predictive. For all cases we discover a cold posterior effect. We test different prior variances and the effect gets stronger with higher variances. The test NLL decreases with colder posteriors up to the point where the classifier is essentially deterministic. The $\mathcal{B}_{\mathrm{original}}$ bound correlates tightly with this behaviour. For increasing values of $\sigma^2_{\hat{\rho}}(\lambda)$ the posterior needs to be colder ($\lambda$ larger) to achieve the same NLL. We see that the bound also tracks this behaviour.

The results for all graphs are averaged over 10 different initializations to the neural network.

\subsection{MNIST-10 experiments}
We repeat the above experiment for the MNIST-10 dataset. As before we train the networks on $Z_{\mathrm{train}}$ using SGD with stepsize $\eta=10^{-3}$ for 10 epochs. We then estimate the bound using $Z_{\mathrm{trainsuffix}}$. The behaviour is similar to the regression datasets, we again find a consistent cold-posterior effect, which the PAC-Bayes bound captures. Also we again see that the effect gets stronger with higher variances. We plot the results in Figure \ref{exp_fig:figure4}. In the Appendix we discuss two additional metrics, the Expected Calibration Error and the Zero-One loss. 

The results for all graphs are averaged over 10 different initializations to the neural network.

An important caveat that we want to disclose is that even for extremely large $\lambda$ the results of the stochastic classifiers and the deterministic ones do not match exactly. We have double checked the relevant code, and believe that the source is some numerical instabilities of the torch package. It is for this reason that we also do not plot the deterministic results on the same graphs as the stochastic ones.

\section{Discussion}
For the case of isotropic Laplace approximations to the posterior, we presented a PAC-Bayesian interpretation of the cold posterior effect. We argued that performance on out-of-sample data is best described through a PAC-Bayes bound, which naturally includes a temperature parameter $\lambda$ which is not constrained to be $\lambda=1$. There are a number of avenues for future work. The most interesting to the authors is extending this analysis to more complex approximations to the posterior and in particular Monte Carlo methods.

\clearpage

\bibliography{neurips_2022}

\begin{thebibliography}{50}
\providecommand{\natexlab}[1]{#1}
\providecommand{\url}[1]{\texttt{#1}}
\expandafter\ifx\csname urlstyle\endcsname\relax
  \providecommand{\doi}[1]{doi: #1}\else
  \providecommand{\doi}{doi: \begingroup \urlstyle{rm}\Url}\fi

\bibitem[kag()]{kaggle}
Kaggle.
\newblock \url{https://www.kaggle.com/datasets/harlfoxem/housesalesprediction}.

\bibitem[Adlam et~al.(2020)Adlam, Snoek, and Smith]{adlam2020cold}
B.~Adlam, J.~Snoek, and S.~L. Smith.
\newblock Cold posteriors and aleatoric uncertainty.
\newblock \emph{arXiv preprint arXiv:2008.00029}, 2020.

\bibitem[Aitchison(2020)]{aitchison2020statistical}
L.~Aitchison.
\newblock A statistical theory of cold posteriors in deep neural networks.
\newblock \emph{arXiv preprint arXiv:2008.05912}, 2020.

\bibitem[Alquier et~al.(2016)Alquier, Ridgway, and
  Chopin]{alquier2016properties}
P.~Alquier, J.~Ridgway, and N.~Chopin.
\newblock {On the properties of variational approximations of Gibbs
  posteriors}.
\newblock \emph{The Journal of Machine Learning Research}, 17\penalty0
  (1):\penalty0 8374--8414, 2016.

\bibitem[Ashukha et~al.(2019)Ashukha, Lyzhov, Molchanov, and
  Vetrov]{ashukha2019pitfalls}
A.~Ashukha, A.~Lyzhov, D.~Molchanov, and D.~Vetrov.
\newblock Pitfalls of in-domain uncertainty estimation and ensembling in deep
  learning.
\newblock In \emph{International Conference on Learning Representations}, 2019.

\bibitem[B{\'e}gin et~al.(2016)B{\'e}gin, Germain, Laviolette, and
  Roy]{begin2016pac}
L.~B{\'e}gin, P.~Germain, F.~Laviolette, and J.-F. Roy.
\newblock Pac-bayesian bounds based on the r{\'e}nyi divergence.
\newblock In \emph{Artificial Intelligence and Statistics}, pages 435--444.
  PMLR, 2016.

\bibitem[Blackwell and Dubins(1962)]{blackwell1962merging}
D.~Blackwell and L.~Dubins.
\newblock Merging of opinions with increasing information.
\newblock \emph{The Annals of Mathematical Statistics}, 33\penalty0
  (3):\penalty0 882--886, 1962.

\bibitem[Blundell et~al.(2015)Blundell, Cornebise, Kavukcuoglu, and
  Wierstra]{blundell2015weight}
C.~Blundell, J.~Cornebise, K.~Kavukcuoglu, and D.~Wierstra.
\newblock Weight uncertainty in neural network.
\newblock In \emph{International Conference on Machine Learning}, pages
  1613--1622. PMLR, 2015.

\bibitem[Catoni(2007)]{catoni2007pac}
O.~Catoni.
\newblock \emph{{PAC-Bayesian supervised classification: the thermodynamics of
  statistical learning}}, volume~56 of \emph{Monograph Series}.
\newblock Institute of Mathematical Statistics Lecture Notes, 2007.

\bibitem[Cohen et~al.(2017)Cohen, Afshar, Tapson, and
  Van~Schaik]{cohen2017emnist}
G.~Cohen, S.~Afshar, J.~Tapson, and A.~Van~Schaik.
\newblock Emnist: Extending mnist to handwritten letters.
\newblock In \emph{2017 international joint conference on neural networks
  (IJCNN)}, pages 2921--2926. IEEE, 2017.

\bibitem[Collier et~al.(2021)Collier, Mustafa, Kokiopoulou, Jenatton, and
  Berent]{collier2021correlated}
M.~Collier, B.~Mustafa, E.~Kokiopoulou, R.~Jenatton, and J.~Berent.
\newblock Correlated input-dependent label noise in large-scale image
  classification.
\newblock In \emph{Proceedings of the IEEE/CVF Conference on Computer Vision
  and Pattern Recognition}, pages 1551--1560, 2021.

\bibitem[Cover(1999)]{cover1999elements}
T.~M. Cover.
\newblock \emph{Elements of information theory}.
\newblock John Wiley \& Sons, 1999.

\bibitem[Dangel et~al.(2019)Dangel, Kunstner, and Hennig]{dangel2019backpack}
F.~Dangel, F.~Kunstner, and P.~Hennig.
\newblock Backpack: Packing more into backprop.
\newblock In \emph{International Conference on Learning Representations}, 2019.

\bibitem[Daxberger et~al.(2021)Daxberger, Kristiadi, Immer, Eschenhagen, Bauer,
  and Hennig]{daxberger2021laplace}
E.~Daxberger, A.~Kristiadi, A.~Immer, R.~Eschenhagen, M.~Bauer, and P.~Hennig.
\newblock Laplace redux-effortless bayesian deep learning.
\newblock \emph{Advances in Neural Information Processing Systems}, 34, 2021.

\bibitem[Deng(2012)]{deng2012mnist}
L.~Deng.
\newblock The mnist database of handwritten digit images for machine learning
  research.
\newblock \emph{IEEE Signal Processing Magazine}, 29\penalty0 (6):\penalty0
  141--142, 2012.

\bibitem[Deshpande et~al.(2021)Deshpande, Achille, Ravichandran, Li, Zancato,
  Fowlkes, Bhotika, Soatto, and Perona]{deshpande2021linearized}
A.~Deshpande, A.~Achille, A.~Ravichandran, H.~Li, L.~Zancato, C.~Fowlkes,
  R.~Bhotika, S.~Soatto, and P.~Perona.
\newblock A linearized framework and a new benchmark for model selection for
  fine-tuning.
\newblock \emph{arXiv preprint arXiv:2102.00084}, 2021.

\bibitem[Dua and Graff(2017)]{Dua:2019}
D.~Dua and C.~Graff.
\newblock {UCI} machine learning repository, 2017.
\newblock URL \url{http://archive.ics.uci.edu/ml}.

\bibitem[Dusenberry et~al.(2020)Dusenberry, Jerfel, Wen, Ma, Snoek, Heller,
  Lakshminarayanan, and Tran]{dusenberry2020efficient}
M.~Dusenberry, G.~Jerfel, Y.~Wen, Y.~Ma, J.~Snoek, K.~Heller,
  B.~Lakshminarayanan, and D.~Tran.
\newblock {Efficient and scalable Bayesian neural nets with rank-1 factors}.
\newblock In \emph{International Conference of Machine Learning}, pages
  2782--2792. PMLR, 2020.

\bibitem[Dziugaite and Roy(2017)]{dziugaite2017computing}
G.~K. Dziugaite and D.~M. Roy.
\newblock Computing nonvacuous generalization bounds for deep (stochastic)
  neural networks with many more parameters than training data.
\newblock \emph{Uncertainty in Artificial Intelligence}, 2017.

\bibitem[Dziugaite et~al.(2021)Dziugaite, Hsu, Gharbieh, Arpino, and
  Roy]{dziugaite2021role}
G.~K. Dziugaite, K.~Hsu, W.~Gharbieh, G.~Arpino, and D.~Roy.
\newblock On the role of data in pac-bayes.
\newblock In \emph{International Conference on Artificial Intelligence and
  Statistics}, pages 604--612. PMLR, 2021.

\bibitem[Fortuin et~al.(2021)Fortuin, Garriga-Alonso, Wenzel, R{\"a}tsch,
  Turner, van~der Wilk, and Aitchison]{fortuin2021bayesian}
V.~Fortuin, A.~Garriga-Alonso, F.~Wenzel, G.~R{\"a}tsch, R.~Turner, M.~van~der
  Wilk, and L.~Aitchison.
\newblock Bayesian neural network priors revisited.
\newblock In \emph{International Conference on Learning Representations}, 2021.

\bibitem[Germain et~al.(2016)Germain, Bach, Lacoste, and
  Lacoste-Julien]{germain2016pac}
P.~Germain, F.~Bach, A.~Lacoste, and S.~Lacoste-Julien.
\newblock {PAC-Bayesian theory meets Bayesian inference}.
\newblock \emph{Advances in Neural Information Processing Systems}, 29, 2016.

\bibitem[Ghosal et~al.(2000)Ghosal, Ghosh, and Van
  Der~Vaart]{ghosal2000convergence}
S.~Ghosal, J.~K. Ghosh, and A.~W. Van Der~Vaart.
\newblock Convergence rates of posterior distributions.
\newblock \emph{Annals of Statistics}, pages 500--531, 2000.

\bibitem[Gr{\"u}nwald(2012)]{grunwald2012safe}
P.~Gr{\"u}nwald.
\newblock {The safe Bayesian}.
\newblock In \emph{International Conference on Algorithmic Learning Theory},
  pages 169--183. Springer, 2012.

\bibitem[Gr{\"u}nwald and Langford(2007)]{grunwald2007suboptimal}
P.~Gr{\"u}nwald and J.~Langford.
\newblock {Suboptimal behavior of Bayes and MDL in classification under
  misspecification}.
\newblock \emph{Machine Learning}, 66\penalty0 (2):\penalty0 119--149, 2007.

\bibitem[Gr{\"u}nwald and Van~Ommen(2017)]{grunwald2017inconsistency}
P.~Gr{\"u}nwald and T.~Van~Ommen.
\newblock {Inconsistency of Bayesian inference for misspecified linear models,
  and a proposal for repairing it}.
\newblock \emph{Bayesian Analysis}, 12\penalty0 (4):\penalty0 1069--1103, 2017.

\bibitem[Immer et~al.(2021)Immer, Korzepa, and Bauer]{immer2021improving}
A.~Immer, M.~Korzepa, and M.~Bauer.
\newblock Improving predictions of bayesian neural nets via local
  linearization.
\newblock In \emph{International Conference on Artificial Intelligence and
  Statistics}, pages 703--711. PMLR, 2021.

\bibitem[Izmailov et~al.(2021)Izmailov, Vikram, Hoffman, and
  Wilson]{izmailov2021bayesian}
P.~Izmailov, S.~Vikram, M.~D. Hoffman, and A.~G.~G. Wilson.
\newblock {What are Bayesian neural network posteriors really like?}
\newblock In \emph{International Conference on Machine Learning}, pages
  4629--4640. PMLR, 2021.

\bibitem[Jacot et~al.(2018)Jacot, Gabriel, and Hongler]{jacot2018neural}
A.~Jacot, F.~Gabriel, and C.~Hongler.
\newblock Neural tangent kernel: Convergence and generalization in neural
  networks.
\newblock \emph{Advances in neural information processing systems}, 31, 2018.

\bibitem[Khan et~al.(2018)Khan, Nielsen, Tangkaratt, Lin, Gal, and
  Srivastava]{khan2018fast}
M.~Khan, D.~Nielsen, V.~Tangkaratt, W.~Lin, Y.~Gal, and A.~Srivastava.
\newblock Fast and scalable bayesian deep learning by weight-perturbation in
  adam.
\newblock In \emph{International Conference on Machine Learning}, pages
  2611--2620. PMLR, 2018.

\bibitem[Khan et~al.(2019)Khan, Immer, Abedi, and Korzepa]{khan2019approximate}
M.~E.~E. Khan, A.~Immer, E.~Abedi, and M.~Korzepa.
\newblock {Approximate inference turns deep networks into Gaussian processes}.
\newblock \emph{Advances in neural information processing systems}, 32, 2019.

\bibitem[Kunstner et~al.(2019)Kunstner, Hennig, and
  Balles]{kunstner2019limitations}
F.~Kunstner, P.~Hennig, and L.~Balles.
\newblock Limitations of the empirical fisher approximation for natural
  gradient descent.
\newblock \emph{Advances in neural information processing systems}, 32, 2019.

\bibitem[Küppers et~al.(2021)Küppers, Kronenberger, Schneider, and
  Haselhoff]{Kueppers_2021_IV}
F.~Küppers, J.~Kronenberger, J.~Schneider, and A.~Haselhoff.
\newblock Bayesian confidence calibration for epistemic uncertainty modelling.
\newblock In \emph{Proceedings of the IEEE Intelligent Vehicles Symposium
  (IV)}, July 2021.

\bibitem[Lecun et~al.(1998)Lecun, Bottou, Bengio, and Haffner]{Lecun726791}
Y.~Lecun, L.~Bottou, Y.~Bengio, and P.~Haffner.
\newblock Gradient-based learning applied to document recognition.
\newblock \emph{Proceedings of the IEEE}, 86\penalty0 (11):\penalty0
  2278--2324, 1998.
\newblock \doi{10.1109/5.726791}.

\bibitem[Lotfi et~al.(2022)Lotfi, Izmailov, Benton, Goldblum, and
  Wilson]{lotfi2022bayesian}
S.~Lotfi, P.~Izmailov, G.~Benton, M.~Goldblum, and A.~G. Wilson.
\newblock Bayesian model selection, the marginal likelihood, and
  generalization.
\newblock \emph{arXiv preprint arXiv:2202.11678}, 2022.

\bibitem[Lyle et~al.(2020)Lyle, Schut, Ru, Gal, and van~der
  Wilk]{lyle2020bayesian}
C.~Lyle, L.~Schut, R.~Ru, Y.~Gal, and M.~van~der Wilk.
\newblock {A Bayesian perspective on training speed and model selection}.
\newblock \emph{Advances in Neural Information Processing Systems},
  33:\penalty0 10396--10408, 2020.

\bibitem[Maddox et~al.(2021)Maddox, Tang, Moreno, Wilson, and
  Damianou]{maddox2021fast}
W.~Maddox, S.~Tang, P.~Moreno, A.~G. Wilson, and A.~Damianou.
\newblock Fast adaptation with linearized neural networks.
\newblock In \emph{International Conference on Artificial Intelligence and
  Statistics}, pages 2737--2745. PMLR, 2021.

\bibitem[Masegosa(2020)]{masegosa2019learning}
A.~Masegosa.
\newblock Learning under model misspecification: Applications to variational
  and ensemble methods.
\newblock \emph{Advances in Neural Information Processing Systems},
  33:\penalty0 5479--5491, 2020.

\bibitem[McAllester(1999)]{mcallester1999some}
D.~A. McAllester.
\newblock {Some PAC-Bayesian theorems}.
\newblock \emph{Machine Learning}, 37\penalty0 (3):\penalty0 355--363, 1999.

\bibitem[Mishkin et~al.(2018)Mishkin, Kunstner, Nielsen, Schmidt, and
  Khan]{mishkin2018slang}
A.~Mishkin, F.~Kunstner, D.~Nielsen, M.~Schmidt, and M.~E. Khan.
\newblock Slang: Fast structured covariance approximations for bayesian deep
  learning with natural gradient.
\newblock \emph{Advances in Neural Information Processing Systems}, 31, 2018.

\bibitem[Morningstar et~al.(2022)Morningstar, Alemi, and
  Dillon]{morningstar2022pacm}
W.~R. Morningstar, A.~Alemi, and J.~V. Dillon.
\newblock {PAC$^m$-Bayes: Narrowing the empirical risk gap in the misspecified
  Bayesian regime}.
\newblock In \emph{International Conference on Artificial Intelligence and
  Statistics}, pages 8270--8298. PMLR, 2022.

\bibitem[Nabarro et~al.(2021)Nabarro, Ganev, Garriga-Alonso, Fortuin, van~der
  Wilk, and Aitchison]{nabarro2021data}
S.~Nabarro, S.~Ganev, A.~Garriga-Alonso, V.~Fortuin, M.~van~der Wilk, and
  L.~Aitchison.
\newblock Data augmentation in bayesian neural networks and the cold posterior
  effect.
\newblock \emph{arXiv preprint arXiv:2106.05586}, 2021.

\bibitem[Naeini et~al.(2015)Naeini, Cooper, and
  Hauskrecht]{naeini2015obtaining}
M.~P. Naeini, G.~Cooper, and M.~Hauskrecht.
\newblock Obtaining well calibrated probabilities using bayesian binning.
\newblock In \emph{Twenty-Ninth AAAI Conference on Artificial Intelligence},
  2015.

\bibitem[Noci et~al.(2021)Noci, Roth, Bachmann, Nowozin, and
  Hofmann]{noci2021disentangling}
L.~Noci, K.~Roth, G.~Bachmann, S.~Nowozin, and T.~Hofmann.
\newblock Disentangling the roles of curation, data-augmentation and the prior
  in the cold posterior effect.
\newblock \emph{Advances in Neural Information Processing Systems}, 34, 2021.

\bibitem[Paszke et~al.(2019)Paszke, Gross, Massa, Lerer, Bradbury, Chanan,
  Killeen, Lin, Gimelshein, Antiga, Desmaison, Kopf, Yang, DeVito, Raison,
  Tejani, Chilamkurthy, Steiner, Fang, Bai, and Chintala]{NEURIPS2019_9015}
A.~Paszke, S.~Gross, F.~Massa, A.~Lerer, J.~Bradbury, G.~Chanan, T.~Killeen,
  Z.~Lin, N.~Gimelshein, L.~Antiga, A.~Desmaison, A.~Kopf, E.~Yang, Z.~DeVito,
  M.~Raison, A.~Tejani, S.~Chilamkurthy, B.~Steiner, L.~Fang, J.~Bai, and
  S.~Chintala.
\newblock Pytorch: An imperative style, high-performance deep learning library.
\newblock In H.~Wallach, H.~Larochelle, A.~Beygelzimer, F.~d\textquotesingle
  Alch\'{e}-Buc, E.~Fox, and R.~Garnett, editors, \emph{Advances in Neural
  Information Processing Systems 32}, pages 8024--8035. Curran Associates,
  Inc., 2019.
\newblock URL
  \url{http://papers.neurips.cc/paper/9015-pytorch-an-imperative-style-high-performance-deep-learning-library.pdf}.

\bibitem[Ritter et~al.(2018)Ritter, Botev, and Barber]{ritter2018scalable}
H.~Ritter, A.~Botev, and D.~Barber.
\newblock {A scalable Laplace approximation for neural networks}.
\newblock In \emph{6th International Conference on Learning Representations},
  volume~6. International Conference on Representation Learning, 2018.

\bibitem[Ru et~al.(2021)Ru, Lyle, Schut, Fil, van~der Wilk, and
  Gal]{ru2020speedy}
R.~Ru, C.~Lyle, L.~Schut, M.~Fil, M.~van~der Wilk, and Y.~Gal.
\newblock Speedy performance estimation for neural architecture search.
\newblock \emph{Advances in Neural Information Processing Systems}, 34, 2021.

\bibitem[Wenzel et~al.(2020)Wenzel, Roth, Veeling, Swiatkowski, Tran, Mandt,
  Snoek, Salimans, Jenatton, and Nowozin]{wenzel2020good}
F.~Wenzel, K.~Roth, B.~S. Veeling, J.~Swiatkowski, L.~Tran, S.~Mandt, J.~Snoek,
  T.~Salimans, R.~Jenatton, and S.~Nowozin.
\newblock {How good is the Bayes posterior in deep neural networks really?}
\newblock \emph{International Conference on Machine Learning}, 2020.

\bibitem[Zancato et~al.(2020)Zancato, Achille, Ravichandran, Bhotika, and
  Soatto]{zancato2020predicting}
L.~Zancato, A.~Achille, A.~Ravichandran, R.~Bhotika, and S.~Soatto.
\newblock Predicting training time without training.
\newblock \emph{Advances in Neural Information Processing Systems},
  33:\penalty0 6136--6146, 2020.

\bibitem[Zeno et~al.(2020)Zeno, Golan, Pakman, and Soudry]{zeno2020cold}
C.~Zeno, I.~Golan, A.~Pakman, and D.~Soudry.
\newblock {Why cold posteriors? on the suboptimal generalization of optimal
  Bayes estimates}.
\newblock In \emph{Third Symposium on Advances in Approximate Bayesian
  Inference}, 2020.

\end{thebibliography}
\bibliographystyle{abbrvnat}

\medskip

\clearpage

\appendix

\section{Proofs of main results}

\subsection{Proof of Proposition 1}

Recall that we model our predictor as $f_{\mathrm{lin}}(\bx;\bw)=f(\bx;\bw_{\hat{\rho}})-\nabla_{\bw}f(\bx;\bw_{\hat{\rho}})^{\top}(\bw-\bw_{\hat{\rho}})$. Then for the choice of a Gaussian likelihood, given a training signal $\bx$, a training label $y$ and weights $\bw$, the negative log-likelihood loss takes the form $\ell_{\mathrm{nll}}(\bw,\bx,y)=\frac{1}{2}\ln(2\pi)+\frac{1}{2}(y-f(\bx;\bw_{\hat{\rho}})-\nabla_{\bw}f(\bx;\bw_{\hat{\rho}})^{\top}(\bw-\bw_{\hat{\rho}}))^2$. We also define $\hat{\mathcal{L}}^{\ell}_{X,Y}(f)=(1/n)\sum_i\ell(f,\bx_i,y_i)$. Our derivations closely follow the approach of \cite{germain2016pac} p.11, section A.4. 

Given the above definitions and modelling choices we develop the empirical risk term
\begin{align*}
2n&\bE_{\bw\sim\hat{\rho}}\hat{\mathcal{L}}_{X,Y}^{\ell_{\mathrm{nll}}}(\bw)-n\ln(2\pi)
= \bE_{\bw\sim\hat{\rho}}\sum_{i=1}^n(y_i-f(\bx_i;\bw_{\hat{\rho}})-\nabla_{\bw}f(\bx_i;\bw_{\hat{\rho}})^{\top}(\bw-\bw_{\hat{\rho}}))^2\\
&= \bE_{\bw\sim\hat{\rho}}\Vert\by-f(\bX;\bw_{\hat{\rho}})-\nabla_{\bw}f(\bX;\bw_{\hat{\rho}})^{\top}(\bw-\bw_{\hat{\rho}})\Vert^2_2 \\
&= \bE_{\bw\sim\hat{\rho}}[\Vert\by-f(\bX;\bw_{\hat{\rho}})\Vert^2_2-2(\by-f(\bX;\bw_{\hat{\rho}}))^{\top}\nabla_{\bw}f(\bX;\bw_{\hat{\rho}})^{\top}(\bw-\bw_{\hat{\rho}})\\ 
&\qquad+ (\bw-\bw_{\hat{\rho}})^{\top}\nabla_{\bw}f(\bX;\bw_{\hat{\rho}})\nabla_{\bw}f(\bX;\bw_{\hat{\rho}})^{\top}(\bw-\bw_{\hat{\rho}})]\\
&= \bE_{\bw\sim\hat{\rho}}[\Vert\by-f(\bX;\bw_{\hat{\rho}})\Vert^2_2-2(\by-f(\bX;\bw_{\hat{\rho}}))^{\top}\nabla_{\bw}f(\bX;\bw_{\hat{\rho}})^{\top}(\bw-\bw_{\hat{\rho}})\\ 
&\qquad+ (\bw-\bw_{\hat{\rho}})^{\top}\textcolor{blue}{\left[\textstyle\sum_i\nabla_{\bw}f(\bx_i;\bw_{\hat{\rho}})\nabla_{\bw}f(\bx_i;\bw_{\hat{\rho}})^{\top}\right]}(\bw-\bw_{\hat{\rho}})]\\
&= \bE_{\bw\sim\hat{\rho}}[\Vert\by-f(\bX;\bw_{\hat{\rho}})\Vert^2_2]-2(\by-f(\bX;\bw_{\hat{\rho}}))^{\top}\nabla_{\bw}f(\bX;\bw_{\hat{\rho}})^{\top}\textcolor{red}{\cancel{\bE_{\bw\sim\hat{\rho}}[\bw-\bw_{\hat{\rho}}]}}\\ 
&\qquad+ \bE_{\bw\sim\hat{\rho}}\left[(\bw-\bw_{\hat{\rho}})^{\top}\left[\textstyle\sum_i\nabla_{\bw}f(\bx_i;\bw_{\hat{\rho}})\nabla_{\bw}f(\bx_i;\bw_{\hat{\rho}})^{\top}\right](\bw-\bw_{\hat{\rho}})\right]\\
&= \Vert\by-f(\bX;\bw_{\hat{\rho}})\Vert^2_2+ \sigma_{\hat{\rho}}^2\left[\textstyle\sum_{i}\textstyle\sum_{j}(\nabla_{\bw}f(\bx_i;\bw_{\hat{\rho}})_j)^2\right]\\
&= \Vert\by-f(\bX;\bw_{\hat{\rho}})\Vert^2_2+ \sigma_{\hat{\rho}}^2h.
\end{align*}
In the penultimate line, we have used the fact that a real number is the trace of itself as well as the cyclic property of the trace. The second summation ($\sum_j$ over the parameters of the model) results from the fact that $\hat{\rho}=\mathcal{N}(\bw_{\hat{\rho}},\sigma^2_{\hat{\rho}}\bI)$ is isotropic with a common scaling factor $\sigma^2_{\hat{\rho}}$. The term in blue is exactly the Gauss--Newton approximation to the Hessian of the full neural network, for the squared loss function \cite{kunstner2019limitations,immer2021improving}, and in the last line we set $h = \left[\sum_{i}\sum_{j}(\nabla_{\bw}f(\bx_i;\bw_{\hat{\rho}})_j)^2\right]$. Since $h$ is a sum of positive numbers, taking into account that the blue term is the Gauss--Newton approximation to the Hessian and if we assume that the Gauss--Newton approximation is diagonal, then $h$ is a measure of the curvature at minimum $\bw_{\hat{\rho}}$ of the loss landscape. We finally get
\begin{equation*}
    \bE_{\bw\sim\hat{\rho}}\hat{\mathcal{L}}_{X,Y}^{\ell_{\mathrm{nll}}}(\bw) = \frac{\Vert\by-f(\bX;\bw_{\hat{\rho}})\Vert^2_2}{2n}+\frac{ \sigma_{\hat{\rho}}^2h}{2n}
+\frac{1}{2}\ln(2\pi).
\end{equation*}
We continue with the KL term which is known to have the following analytical expression for  Gaussian prior and posterior distributions
\begin{align*}
\mathrm{KL}(\mathcal{N}(\bw_{\hat{\rho}},\sigma^2_{\hat{\rho}}\bI)\Vert\mathcal{N}(\bw_{\pi},\sigma^2_{\pi}\bI))
 =\frac{1}{2}\left(d\frac{\sigma^2_{\hat{\rho}}}{\sigma_{\pi}^2} + \frac{1}{\sigma^2_{\pi}}\Vert\bw_{\hat{\rho}}-\bw_{\pi}\Vert^2 -d-d \ln\frac{\sigma^2_{\hat{\rho}}}{\sigma_{\pi}^2} \right).
\end{align*}
We finally develop the moment term. Using an intermediate variable $\lambda_n=\frac{\lambda n}{2}$ to simplify the calculations, we get

\begin{align*}
\Psi_{\ell,\pi,\mathcal{D}}&(\lambda,n)=
\ln \bE_{f\sim\pi}\bE_{(X',Y')\sim\mathcal{D}^n}\exp \left[\lambda n\left(\mathcal{L}_{\mathcal{D}}^{\ell_{\mathrm{nll}}}(f)- \hat{\mathcal{L}}_{X',Y'}^{\ell_{\mathrm{nll}}}(f)\right)  \right]\\
&=\ln \bE_{f\sim\pi}\bE_{(X',Y')\sim\mathcal{D}^n}\exp \left[\lambda_n \left(\bE_{(\bx,y)}\left[\ln(2\pi)+(y-f_{\mathrm{lin}}(\bx;\bw)^2\right] \right.\right.\\
&\qquad - \left.\left. \ln(2\pi)-(1/n)\textstyle\sum_i(y_i-f_{\mathrm{lin}}(\bx_i;\bw)^2 \right)  \right]\\
&=\ln \bE_{f\sim\pi}\bE_{(X',Y')\sim\mathcal{D}^n}\exp \left[\lambda_n \left(\bE_{(\bx,y)}\left[(y-f_{\mathrm{lin}}(\bx;\bw)^2\right]-(1/n)\textstyle\sum_i(y_i-f_{\mathrm{lin}}(\bx_i;\bw)^2 \right)  \right]\\
&\leq\ln \bE_{\bw\sim\pi}\exp \left[\lambda_n\bE_{(\bx,y)}\left(y- f_{\mathrm{lin}}(\bx;\bw)\right)^2  \right]\\
&=\ln \bE_{\bw\sim\pi}\exp [\lambda_n\bE_{(\bx,y)}(f(\bx;\bw_{\hat{\rho}})+\nabla_{\bw}f(\bx;\bw_{\hat{\rho}})^{\top}(\bw_*-\bw_{\hat{\rho}})+\epsilon\\
&\qquad - (f(\bx;\bw_{\hat{\rho}})+\nabla_{\bw}f(\bx;\bw_{\hat{\rho}})^{\top}(\bw-\bw_{\hat{\rho}})))^2]\\
&=\ln \bE_{\bw\sim\pi}\exp [\lambda_n\bE_{(\bx,y)}(\nabla_{\bw}f(\bx;\bw_{\hat{\rho}})^{\top}(\bw_*-\bw)+\epsilon)^2]\\
&=\ln \bE_{\bw\sim\pi}\exp [\lambda_n\bE_{\bx}[(\nabla_{\bw}f(\bx;\bw_{\hat{\rho}})^{\top}(\bw_*-\bw))^2]+\lambda_n\sigma_{\epsilon}^2].
\end{align*}

Inequality in line 4 is because the exponential function is less than 1 on the negative half line. In the fifth line we use our modelling choice $y = f(\bx;\bw_{\hat{\rho}})+\nabla_{\bw}f(\bx;\bw_{\hat{\rho}})^{\top}(\bw_*-\bw_{\hat{\rho}})+\epsilon$, where $\epsilon\sim\mathcal{N}(0,\sigma_{\epsilon}^2)$. To obtain the final line we note that the gradient of the \emph{neural network output} with respect to $\bw$, that is $\nabla_{\bw}f(\bx;\bw_{\hat{\rho}})$, does \emph{not} depend on the label $y$. We get the last line by applying the square and taking the expectation, given that the noise $\epsilon$ is centered.

We now take into account the Gaussian mixture modelling for the gradients per data sample, $\nabla_{\bw}f(\bx;\bw_{\hat{\rho}})\sim\sum_{j=1}^k\phi_j\mathcal{N}(\bmu_j,\sigma_{\bx j}^2\bI)$. We get

\begin{align*}
\bE_{\bx}&[(\nabla_{\bw}f(\bx;\bw_{\hat{\rho}})^{\top}(\bw_*-\bw))^2]
=\bE_{\bx}[(\textstyle\sum_i\nabla_{\bw}f(\bx;\bw_{\hat{\rho}})_i(\bw_*-\bw)_i)^2]\\
&=\bE_{\bx}\left[(\textstyle\sum_i\nabla_{\bw}f(\bx;\bw_{\hat{\rho}})_i^2(\bw_*-\bw)_i^2
+\textcolor{red}{2\textstyle\sum_{i,j}\nabla_{\bw}f(\bx;\bw_{\hat{\rho}})_i\nabla_{\bw}f(\bx;\bw_{\hat{\rho}})_j(\bw_*-\bw)_i(\bw_*-\bw)_j)}\right]\\
&=\textstyle\sum_i\bE_{\bx}[\nabla_{\bw}f(\bx;\bw_{\hat{\rho}})_i^2](\bw_*-\bw)_i^2
=\textstyle\sum_i\textcolor{blue}{\textstyle\sum_{j=1}^k(\phi_j\sigma^2_{\bx j})}(\bw_*-\bw)_i^2
=\textcolor{blue}{\sigma_{\bx}^2}\Vert\bw_*-\bw\Vert_2^2.
\end{align*}

The red term cancels out because we assumed that each weight is independent from the others. Next we use the Gaussian mixture modelling to get $\bE_{\bx}[\nabla_{\bw}f(\bx;\bw_{\hat{\rho}})_i^2]=\sum_{j=1}^k(\phi_j\sigma^2_{\bx j})$, and we finally set $\sigma_{\bx}^2=\sum_{j=1}^k(\phi_j\sigma^2_{\bx j})$, as each component of the mixture is isotropic, thus the second moment of all weights is the same. 
By completing the square above, one obtains the Gaussian expectation of this squared norm and forms the moment term  as follows

\begin{align*}
\Psi_{\ell,\pi,\mathcal{D}}&(\lambda,n)=\ln \bE_{\bw\sim\pi}\exp \left[\lambda_n \textcolor{blue}{\sigma_{\bx}^2}\Vert\bw_*-\bw\Vert_2^2 +\lambda_n\sigma_{\epsilon}^2\right] \\
&=\ln \left(\frac{1}{(1-2\lambda_n\sigma_{\bx}^2\sigma_{\pi}^2)^{\frac{d}{2}}} \exp \left[\frac{\lambda_n\sigma_{\bx}^2\Vert\bw_*-\bw_{\pi}\Vert_2^2}{1-2\lambda_n\sigma_{\bx}^2\sigma_{\pi}^2}+\lambda_n\sigma_{\epsilon}^2\right]\right)\\
&=-\frac{d}{2}\ln(1-2\lambda_n\sigma_{\bx}^2\sigma_{\pi}^2) + \frac{\lambda_n\sigma_{\bx}^2\Vert\bw_*-\bw_{\pi}\Vert_2^2}{1-2\lambda_n\sigma_{\bx}^2\sigma_{\pi}^2}+\lambda_n\sigma_{\epsilon}^2\\
&\leq\frac{\lambda_n\sigma_{\bx}^2\sigma_{\pi}^2d}{1-2\lambda_n\sigma_{\bx}^2\sigma_{\pi}^2} + \frac{\lambda_n\sigma_{\bx}^2\Vert\bw_*-\bw_{\pi}\Vert_2^2}{1-2\lambda_n\sigma_{\bx}^2\sigma_{\pi}^2}+\lambda_n\sigma_{\epsilon}^2\\
&=\frac{\lambda_n\sigma_{\bx}^2(\sigma_{\pi}^2d+\Vert\bw_*-\bw_{\pi}\Vert_2^2)}{1-2\lambda_n\sigma_{\bx}^2\sigma_{\pi}^2}+\lambda_n\sigma_{\epsilon}^2,\\
\end{align*}
which assumes $1-2\lambda_n\sigma_{\bx}^2\sigma_{\pi}^2>0$. The second line above is obtained by using the moment generating function of noncentral $\chi^2$ variables, while the inequality comes from $\ln(u)< u-1$ for $u>1$. 
Setting back $\frac{\lambda n}{2}$ in place of $\lambda_n$, we get

\begin{equation*}
    \frac{1}{\lambda n}\Psi_{\ell,\pi,\mathcal{D}}(\lambda,n) \leq \frac{\sigma_{\bx}^2(\sigma_{\pi}^2d+\Vert\bw_*-\bw_{\pi}\Vert_2^2)}{2-2\lambda n 2\sigma_{\bx}^2\sigma_{\pi}^2}+\frac{\sigma_{\epsilon}^2}{2}.
\end{equation*} 

We are now ready to minimize the following objective, where the moment term is absent since it does not depend on $\sigma^2_{\hat{\rho}}$
\begin{equation*}
\min_{\sigma^2_{\hat{\rho}}}  \bE_{\bw\sim\hat{\rho}}\hat{\mathcal{L}}_{X,Y}^{\ell_{\mathrm{nll}}}(\bw) + \frac{1}{\lambda n}\left[\mathrm{KL}(\mathcal{N}(\bw_{\hat{\rho}},\sigma^2_{\hat{\rho}}\bI)\Vert\mathcal{N}(\bw_{\pi},\sigma^2_{\pi}\bI)) +\ln\frac{1}{\delta}\right]
\end{equation*}
The derivative of the objective function w.r.t. $\sigma^2_{\hat{\rho}}$ simply writes

\begin{align*}
\frac{\partial}{\partial \sigma_{\hat{\rho}}^2}&\left(\frac{\Vert\by-f(\bX;\bw_{\hat{\rho}})\Vert^2_2}{2n}+\frac{ \sigma_{\hat{\rho}}^2h}{2n}+\frac{1}{2}\ln(2\pi)\right.\\ 
&\left.+\frac{1}{\lambda n} \left[\frac{1}{2}\left(\frac{1}{\sigma_{\pi}^2} d  \sigma^2_{\hat{\rho}} + \frac{1}{\sigma^2_{\pi}}\Vert\bw_{\hat{\rho}}-\bw_{\pi}\Vert_2^2 -d-d \ln\sigma_{\hat{\rho}}^2+d \ln\sigma_{\pi}^2 \right) +\ln\frac{1}{\delta}\right]\right)\\
&=\frac{h}{2n}+\frac{1}{2\lambda n}\left(\frac{d}{\sigma^2_{\pi}}-\frac{d}{\sigma^2_{\hat{\rho}}} \right).
\end{align*}

Now setting the above to zero we get the typical prior-to-posterior update for a Gaussian precision term $$\frac{1}{\sigma^2_{\hat{\rho}}}=\frac{\lambda h}{d}+\frac{1}{\sigma_{\pi}^2}.$$ 

The proposition is proven by replacing the terms in the bound from Theorem 1 with the results derived above.

\subsection{Proof of Corollary 1}
The result is directly obtained by a simple inspection of Proposition 1.

\section{Experiments}

\subsection{Experimental setup}
We run our experiments on GPUs of the type NVIDIA GeForce RTX2080ti, on our local cluster. The total computation time was approximately 125 GPU hours. In the following list we include the libraries and datasets that we used together with their corresponding licences
\begin{itemize}
\item Laplace-Redux Package \cite{daxberger2021laplace}: MIT License
\item Netcal package \cite{Kueppers_2021_IV}: Apache Software License
\item Pytorch package \cite{NEURIPS2019_9015}: Modified BSD Licence
\item Abalone, Diamonds datasets \cite{Dua:2019}: -
\item KC\_House datasets \cite{kaggle}: CC0, Public Domain
\item MNIST-10 dataset \cite{deng2012mnist}: MIT Licence
\end{itemize}

\subsection{Dataset splits}
In all experiments we will split the dataset into 4 sets: $Z_{\mathrm{train}}$ the training set, $Z_{\mathrm{test}}$ the testing set, $Z_{\mathrm{validation}}$ the validation set, $Z_{\mathrm{true}}$ a large sample set that is used to approximate the complete distribution, and $\mathcal{Z}_{\mathrm{trainsuffix}}$ which is an additional set that we use in a special way when we evaluate our bounds. We detail the use of each split in the following sections; refer to Table~\ref{tab:splits} for the specifics of splits for each dataset.

\begin{table}[h!]
\begin{center}
\begin{tabular}{lccccl}\toprule
           & $Z_{\mathrm{train}}$  & $Z_{\mathrm{test}}$ & $ Z_{\mathrm{validation}}$ & $Z_{\mathrm{trainsuffix}}$  & $Z_{\mathrm{true}}$ \\ \midrule
Abalone    & 751 & 835 & 418 & 84 & 2000 \\
KC\_House & 3923 & 4323 & 2161 & 400 & 10406 \\
Diamonds & 9788 & 10788 & 5394 & 1000 & 25970 \\
MNIST-10   & 50000 & 10000 & 5000 & 5000 & 10000 \\\bottomrule
\end{tabular}
\end{center}
\caption{In this table we detail the number of samples that we add to each set of our split, for each dataset. We aim to have a sufficiently high number of samples for the $Z_{\mathrm{true}}$. $Z_{\mathrm{trainsuffix}}$ is chosen to be approximately 10\% of the $Z_{\mathrm{train}}$ (note that $Z_{\mathrm{trainsuffix}}$ contains new samples). For the regression datasets our training set is approximately the same size as the testing set which is not a common setup in classification. However, our aim is not to obtain the best training and testing error but to investigate the behaviour of our models for varying $\lambda$.\label{tab:splits}
}
\end{table}

\subsection{Models}
For our regression datasets we use a fully connected network with two hidden layers with 100 neurons each and the ReLU non-linearity. We train our networks to minimize the Mean Square Error (MSE) loss. 

For our classification dataset MNIST-10 we use the LeNET \cite{Lecun726791} architecture. We train our networks using the softmax activation and the cross-entropy loss.

\subsection{Evaluation of bounds}

\subsubsection{All approximate bound evaluation $\mathcal{B}_{\mathrm{approximate}}$}
We need to evaluate the following bound

\begin{equation}
\begin{split}
&\bE_{\bw\sim\hat{\rho}}\mathcal{L}_{\mathcal{D}}^{\ell_{\mathrm{nll}}}(\bw) \leq \\
&\textcolor{red}{\underbrace{\frac{\Vert\by-f(\bX;\bw_{\hat{\rho}})\Vert^2_2}{2n}+\left(\frac{1}{\frac{\lambda h}{d}+\frac{1}{\sigma_{\pi}^2}}\right) \frac{h}{2n} +\frac{1}{2}\ln(2\pi)}_{\text{Empirical Risk}}} + \textcolor{Green}{\underbrace{\frac{\sigma_{\bx}^2(\sigma_{\pi}^2d+\Vert\bw_*-\bw_{\pi}\Vert_2^2)}{2-2\lambda n \sigma_{\bx}^2\sigma_{\pi}^2}+\frac{\sigma_{\epsilon}^2}{2}}_{\text{Moment}}}\\
&+\textcolor{blue}{\underbrace{\frac{1}{\lambda n} \left[\frac{1}{2}\left(\frac{d}{\sigma_{\pi}^2}  \frac{1}{\frac{\lambda h}{d}+\frac{1}{\sigma_{\pi}^2}} + \frac{1}{\sigma^2_{\pi}}\Vert\bw_{\hat{\rho}}-\bw_{\pi}\Vert^2_2 -d-d \ln\frac{1}{\frac{\lambda h}{d}+\frac{1}{\sigma_{\pi}^2}}+d \ln\sigma_{\pi}^2 \right) +\ln\frac{1}{\delta}\right]}_{\text{KL}}}. \\
\end{split}
\end{equation}

To estimate the bound we need to measure the following quantities

\begin{itemize}
\item $h = \left[\sum_{i}\sum_{j}(\nabla_{\bw}f(\bx_i;\bw_{\hat{\rho}})_j)^2\right]$ the curvature at the minimum. Note how this corresponds to the the sum of the squared gradients per data sample. We can compute this term using the Laplace-Redux package \cite{daxberger2021laplace} which has as a backend the BackPACK package \cite{dangel2019backpack}.
\item $\bw_{\pi}$ and $\bw_{\hat{\rho}}$ the prior and posterior means. We can also compute these terms explicitly. We typically train a deterministic neural network with SGD on $\mathcal{Z}_{\mathrm{train}}$ to obtain a MAP estimate $\bw_{\pi}$ then we also set $\bw_{\hat{\rho}}=\bw_{\pi}$, such that $\Vert\bw_{\hat{\rho}}-\bw_{\pi}\Vert^2_2=0$ in the KL term. This typically makes bounds tighter and is valid so long as we evaluate the other terms in the bound on $\mathcal{Z}_{\mathrm{trainsuffix}}$. 
\item $\sigma^2_{\bx}$ the per weight variance of the per-sample gradients. We estimate these using the data split reserved for approximating the full distribution $\mathcal{Z}_{\mathrm{true}}$. We estimate this quantity as $\sigma^2_{\bx} = \frac{\sum_{i\in\mathcal{Z}_{\mathrm{true}}}\sum_{j}(\nabla_{\bw}f(\bx_i;\bw_{\hat{\rho}})_j)^2}{\#\mathcal{Z}_{\mathrm{true}}\#\mathrm{weights}}$. Note that we do the above instead of actually fitting a Gaussian mixture on the gradients which would be tedious and error prone.
\item $\Vert\by-f(\bX;\bw_{\hat{\rho}})\Vert^2_2$ the MSE of the MAP classifier.
\item $\sigma^2_{\epsilon}$ the aleatoric uncertainty of the data. While we could estimate this using for example a Gaussian Process, since it is just a small constant we set it to be $\sigma^2_{\epsilon}=1$ in all experiments
\item $\Vert\bw_*-\bw_{\pi}\Vert_2^2$ the $\ell_2$ norm of the difference between the weights of the oracle function that generated the labels $\bw_*$, and our prior mean $\bw_{\pi}$. The oracle quantity $\bw_*$ is unknown. Setting $\Vert\bw_*-\bw_{\pi}\Vert_2^2\approx\Vert\bw_{\hat{\rho}}-\bw_{\pi}\Vert_2^2=\Vert\bw_{\pi}-\bw_{\pi}\Vert_2^2=0$ might be too optimistic so instead we set $\Vert\bw_*-\bw_{\pi}\Vert_2^2\approx\Vert\bw_{\hat{\rho}}\Vert_2^2=\Vert\bw_{\pi}\Vert_2^2$. (Remember that we set $\bw_{\hat{\rho}}=\bw_{\pi}$ to make the bound tighter.)
\end{itemize}

The values of the following variables can be set when evaluating the bound:
\begin{itemize}
\item $\sigma^2_{\pi}$ the prior variance.
\item $d$ the number of weights in the model.
\item $\lambda$ the temperature parameter.
\item $\delta$ the confidence of the bound, we typically use $\delta=0.05$.
\end{itemize}

\subsubsection{Mixed bound evaluation $\mathcal{B}_{\mathrm{mixed}}$}

We need to evaluate the following bound

\begin{equation}
\begin{split}
&\bE_{\bw\sim\hat{\rho}}\mathcal{L}_{\mathcal{D}}^{\ell_{\mathrm{nll}}}(\bw) \leq \\
&\textcolor{red}{\underbrace{\frac{\Vert\by-f(\bX;\bw_{\hat{\rho}})\Vert^2_2}{2n}+\left(\frac{1}{\frac{\lambda h}{d}+\frac{1}{\sigma_{\pi}^2}}\right) \frac{h}{2n} +\frac{1}{2}\ln(2\pi)}_{\text{Empirical Risk}}} \\
&+ \textcolor{Green}{\underbrace{\frac{1}{\lambda n}\ln \bE_{f\sim\pi}\bE_{(X',Y')\sim\mathcal{D}^n}\exp \left[\lambda n\left(\mathcal{L}_{\mathcal{D}}^{\ell_{\mathrm{nll}}}(f)- \hat{\mathcal{L}}_{X',Y'}^{\ell_{\mathrm{nll}}}(f)\right)  \right]}_{\text{Moment}}}\\
&+\textcolor{blue}{\underbrace{\frac{1}{\lambda n} \left[\frac{1}{2}\left(\frac{d}{\sigma_{\pi}^2}  \frac{1}{\frac{\lambda h}{d}+\frac{1}{\sigma_{\pi}^2}} + \frac{1}{\sigma^2_{\pi}}\Vert\bw_{\hat{\rho}}-\bw_{\pi}\Vert^2_2 -d-d \ln\frac{1}{\frac{\lambda h}{d}+\frac{1}{\sigma_{\pi}^2}}+d \ln\sigma_{\pi}^2 \right) +\ln\frac{1}{\delta}\right]}_{\text{KL}}} \\
\end{split}
\end{equation}

To estimate the bound we need to measure the same quantities as in the $\mathcal{B}_{\mathrm{approximate}}$ except for $\sigma^2_{\bx}, \sigma^2_{\epsilon}$ and $\Vert\bw_*-\bw_{\pi}\Vert_2^2$. In their place we need to estimate
$$\Psi_{\ell,\pi,\mathcal{D}}(\lambda,n)=\ln \bE_{f\sim\pi}\bE_{(X',Y')\sim\mathcal{D}^n}\exp \left[\lambda\left(\mathcal{L}_{\mathcal{D}}^{\ell_{\mathrm{nll}}}(f)- \hat{\mathcal{L}}_{X',Y'}^{\ell_{\mathrm{nll}}}(f)\right)  \right].$$ 
We can approximate this term as 
$$\Psi_{\ell,\pi,\mathcal{D}}(\lambda,n) \approx \ln \frac{1}{m}\sum_{f_i\sim\pi}\sum_{(X'_j,Y'_j)\sim\mathcal{D}^n}\exp \left[\lambda\left(\mathcal{L}_{\mathcal{D}}^{\ell_{\mathrm{nll}}}(f_i)- \hat{\mathcal{L}}_{X'_j,Y'_j}^{\ell_{\mathrm{nll}}}(f_i)\right)  \right]$$ 
by using Monte Carlo sampling. We use $m=100$ samples to approximate this term in all experiments. 

\subsubsection{Original bound evaluation $\mathcal{B}_{\mathrm{original}}$}

We need to evaluate the following bound

\begin{equation}
\begin{split}
&\bE_{\bw\sim\hat{\rho}}\mathcal{L}_{\mathcal{D}}^{\ell_{\mathrm{nll}}}(\bw) \leq \textcolor{red}{\underbrace{\bE_{f\sim\hat{\rho}}\hat{\mathcal{L}}_{X,Y}^{\ell_{\mathrm{nll}}}(f)}_{\text{Empirical Risk}}} + \textcolor{Green}{\underbrace{\frac{1}{\lambda n}\ln \bE_{f\sim\pi}\bE_{(X',Y')\sim\mathcal{D}^n}\exp \left[\lambda n\left(\mathcal{L}_{\mathcal{D}}^{\ell_{\mathrm{nll}}  }(f)- \hat{\mathcal{L}}_{X',Y'}^{\ell_{\mathrm{nll}}}(f)\right)  \right]}_{\text{Moment}}}\\
&+\textcolor{blue}{\underbrace{\frac{1}{\lambda n} \left[\frac{1}{2}\left(\frac{d}{\sigma_{\pi}^2}  \frac{1}{\frac{\lambda h}{d}+\frac{1}{\sigma_{\pi}^2}} + \frac{1}{\sigma^2_{\pi}}\Vert\bw_{\hat{\rho}}-\bw_{\pi}\Vert^2_2 -d-d \ln\frac{1}{\frac{\lambda h}{d}+\frac{1}{\sigma_{\pi}^2}}+d \ln\sigma_{\pi}^2 \right) +\ln\frac{1}{\delta}\right]}_{\text{KL}}} \\
\end{split}
\end{equation}

To estimate the bound we need to measure the same quantities as in the $\mathcal{B}_{\mathrm{mixed}}$ except for the empirical risk. We estimate this by sampling directly from the empirical loss using Monte Carlo sampling
$$\bE_{f\sim\hat{\rho}}\hat{\mathcal{L}}_{X,Y}^{\ell}(f)\approx\frac{1}{m}\sum_{f_i\sim\hat{\rho}}\hat{\mathcal{L}}_{X,Y}^{\ell}(f_i).$$ 
We use $m=100$ samples to approximate this term in all experiments.

\subsubsection{Additional notes on bound evaluation}
For the regression datasets We tested 20 different values for $\sigma^2_{\pi}$ in $[0.00001,0.1]$. The effect of increasing $\sigma^2_{\pi}$ is the same as Figure 3 of the main text. With increasing $\sigma^2_{\pi}$ the cold-posterior effect increases and we need higher values of $\lambda$ for optimal performance on the test set. We take the same values of $\sigma^2_{\pi}$ for the classification dataset MNIST-10 and observe the same effect as in the regression datasets. With increasing $\sigma^2_{\pi}$ the cold posterior effect increases.

We try to make our bounds as tight as possible. To do this we try to control the term $\Vert\bw_{\hat{\rho}}-\bw_{\pi} \Vert_2^2$ which typically dominates the bound. We follow for all tasks a variation of the approach of \cite{dziugaite2021role}. Specifically we use $\mathcal{Z}_{\mathrm{train}}$ to learn a prior mean $\bw_{\pi}$ then we set, $\bw_{\hat{\rho}}=\bw_{\pi}$, such that $\Vert\bw_{\hat{\rho}}-\bw_{\pi} \Vert_2^2=0$. Note that we can still evaluate a valid bound so long as we set $(X,Y)$ in Theorem 1 to be independent of the prior mean. This is the reason why we separated a part of the training set in the form of $\mathcal{Z}_{\mathrm{trainsuffix}}$. We thus set $(X,Y) = \mathcal{Z}_{\mathrm{trainsuffix}}$ in Theorem 1. All bounds ($\mathcal{B}_{\mathrm{original}}$, $\mathcal{B}_{\mathrm{mixed}}$, $\mathcal{B}_{\mathrm{approximate}}$) can then be evaluated by taking into account this substitution. Note that our final model deviates from what would be typically used in practice, but it shouldn't deviate significantly. Specifically our models are a modification of the commonly used Laplace approximation \cite{daxberger2021laplace}. For example, in the case of MNIST, we still use 50K samples for training, so our MAP estimate should be similar to typical applications of LeNet on MNIST. We only use $(X,Y) = \mathcal{Z}_{\mathrm{trainsuffix}}$ to learn the \emph{posterior variance} of a Laplace approximation, and in particular to estimate the curvature parameter $h$. This is also why it is difficult to apply our approach to other datasets. The MNIST dataset has a known extended version, EMNIST, from which we can draw extra samples to construct $\mathcal{Z}_{\mathrm{trainsuffix}}$ and $\mathcal{Z}_{\mathrm{true}}$ while still training on the full original training set. For other datasets (such as CIFAR-10) we are not aware of extended versions, and thus we would necessarily have to draw $\mathcal{Z}_{\mathrm{trainsuffix}}$ and $\mathcal{Z}_{\mathrm{true}}$ from $\mathcal{Z}_{\mathrm{train}}$.  

In our experiments we test multiple values of $\lambda$ and $\sigma_{\pi}^2$. Typically one would need to take a union bound over a grid on these parameters so as for the generalization bound to be valid \cite{dziugaite2017computing}. However this typically costs only logarithmically to the actual bound. We ignore these calculations as our bounds are in general quite loose anyway, and these calculations would result in additional terms would make the final bound even more complex. 

For the bounds to be valid, one would typically want to show concentration inequalities such that the Monte Carlo estimates of the Empirical Risk and the Moment terms concentrate close to the true expected value with high probability. We do not provide such guarantees. Note however that, at least for the Empirical Risk term, our sample size of $m=100$ from the posterior distribution over weights is a sample size that is typically used in practice and provides good estimates. Regarding the Moment term we typically use $m=100$. Specifically we sample $10$ samples $\bw_i\sim\pi$ and for each $\bw_i$ we sample $10$ samples from $X_j,Y_j\sim\mathcal{D}$. We have tried to balance sampling sufficiently to approximate the expectation on the one hand, and also not too much such that the computations become prohibitive.

\clearpage

\subsection{Additional regression results}
In Figure \ref{exp_fig_detailed:figure_full} we see more detailed experiments on the regression datasets Abalone, Diamonds and KC\_House. We see that for all the neural networks that we trained across all datasets the $\mathcal{B}_{\mathrm{approximate}}$ bound is very loose. Specifically for the cases we consider $\lambda$ is always restricted to be $\lambda<1$ which is very limiting since we want to investigate cold posteriors and $\lambda>1$. When comparing the $\mathcal{B}_{\mathrm{mixed}}$ and $\mathcal{B}_{\mathrm{original}}$ bounds we see that there is little change in the bound values.  Specifically estimating the Empirical Risk with Monte Carlo sampling, instead of using a Taylor expansion of second order (as in $\mathcal{B}_{\mathrm{approximate}}$) doesn't yield significant benefits. The big improvements are the result of estimating the Moment term using Monte Carlo sampling.

\begin{figure*}[h!]
\centering
\begin{subfigure}{.33\textwidth}
  \centering
  \includegraphics[width=\textwidth]{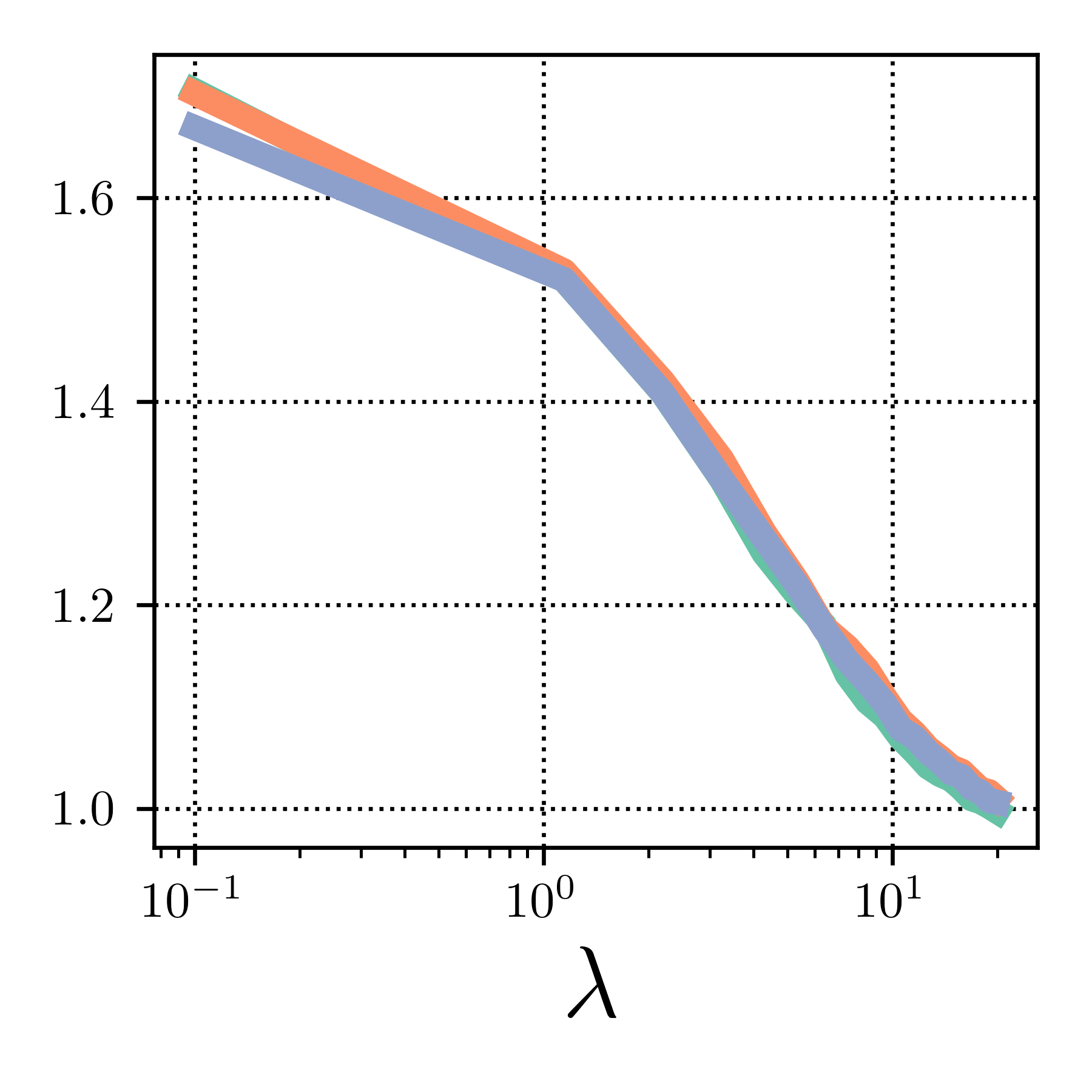}
\end{subfigure}%
\begin{subfigure}{.33\textwidth}
  \centering
  \includegraphics[width=\textwidth]{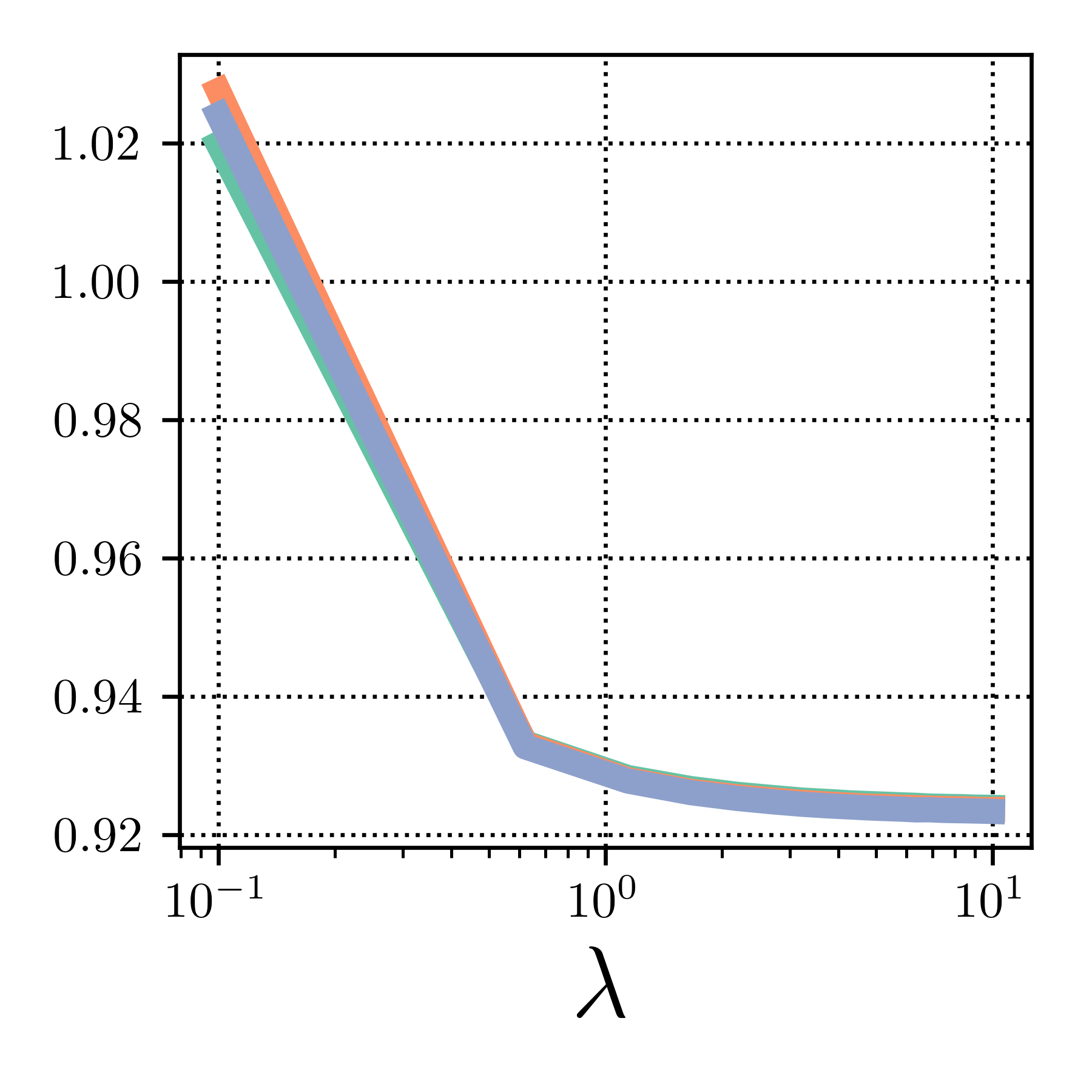}
\end{subfigure}%
\begin{subfigure}{.33\textwidth}
  \centering
  \includegraphics[width=\textwidth]{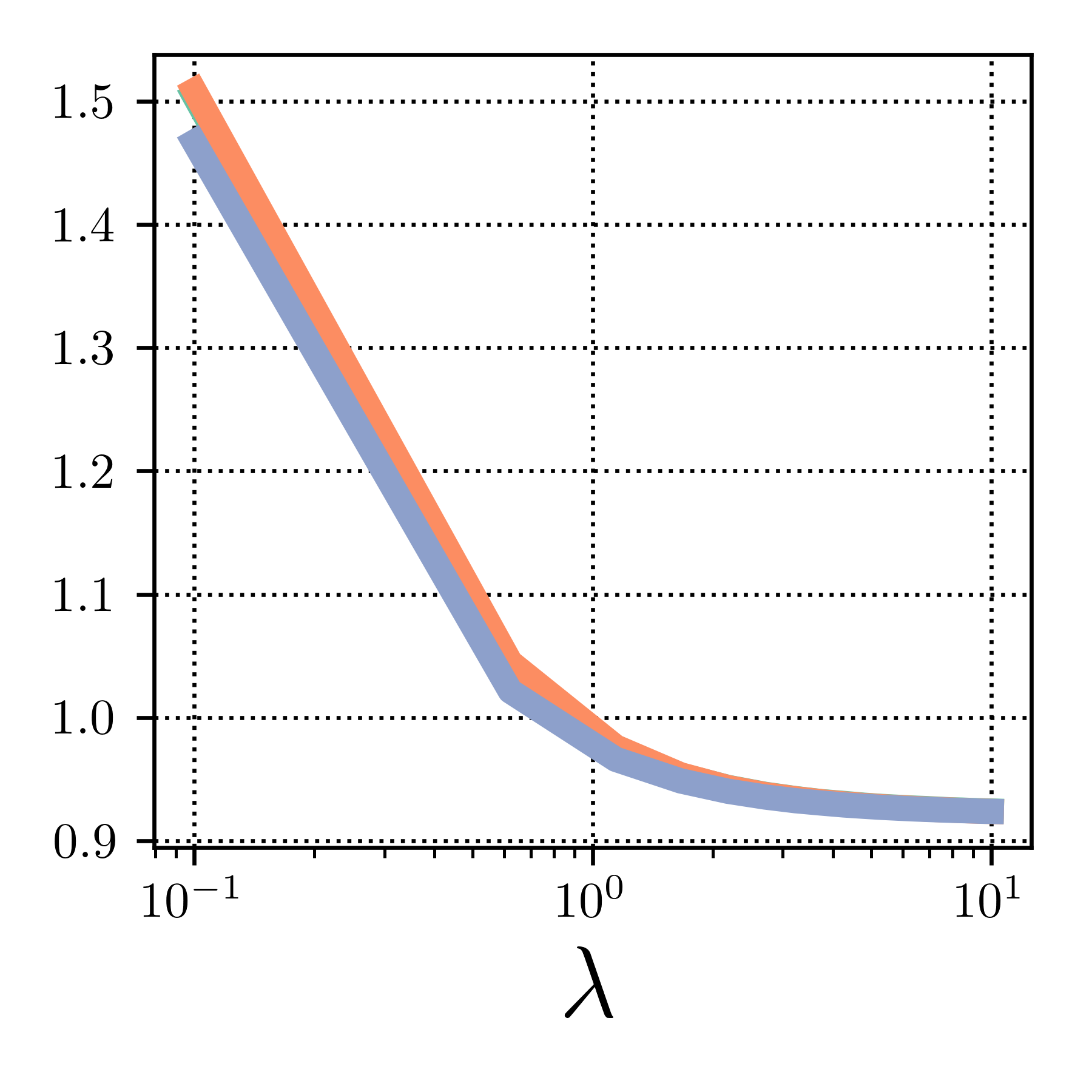}
\end{subfigure}%

\begin{subfigure}{.33\textwidth}
  \centering
  \includegraphics[width=\textwidth]{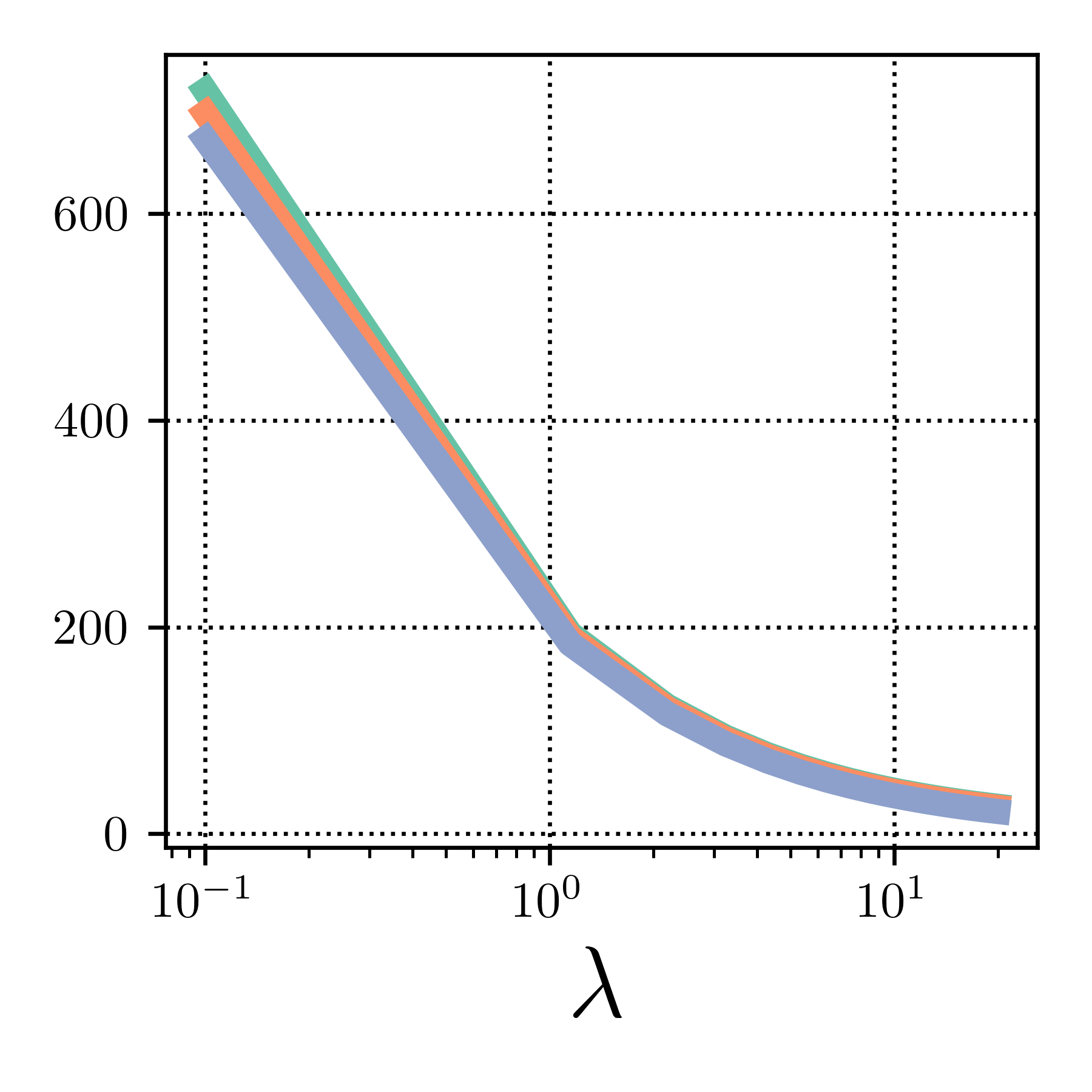}
\end{subfigure}%
\begin{subfigure}{.33\textwidth}
  \centering
  \includegraphics[width=\textwidth]{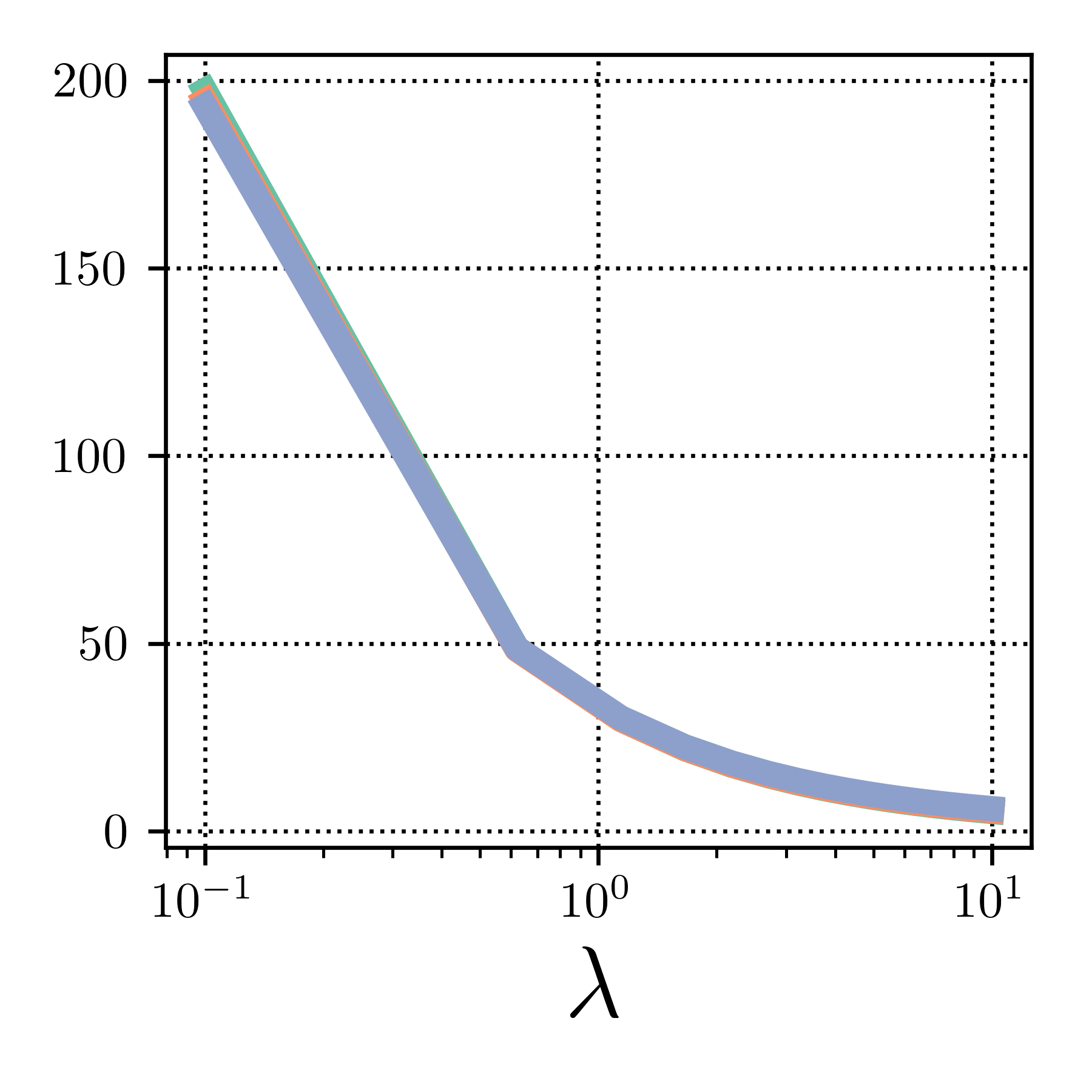}
\end{subfigure}%
\begin{subfigure}{.33\textwidth}
  \centering
  \includegraphics[width=\textwidth]{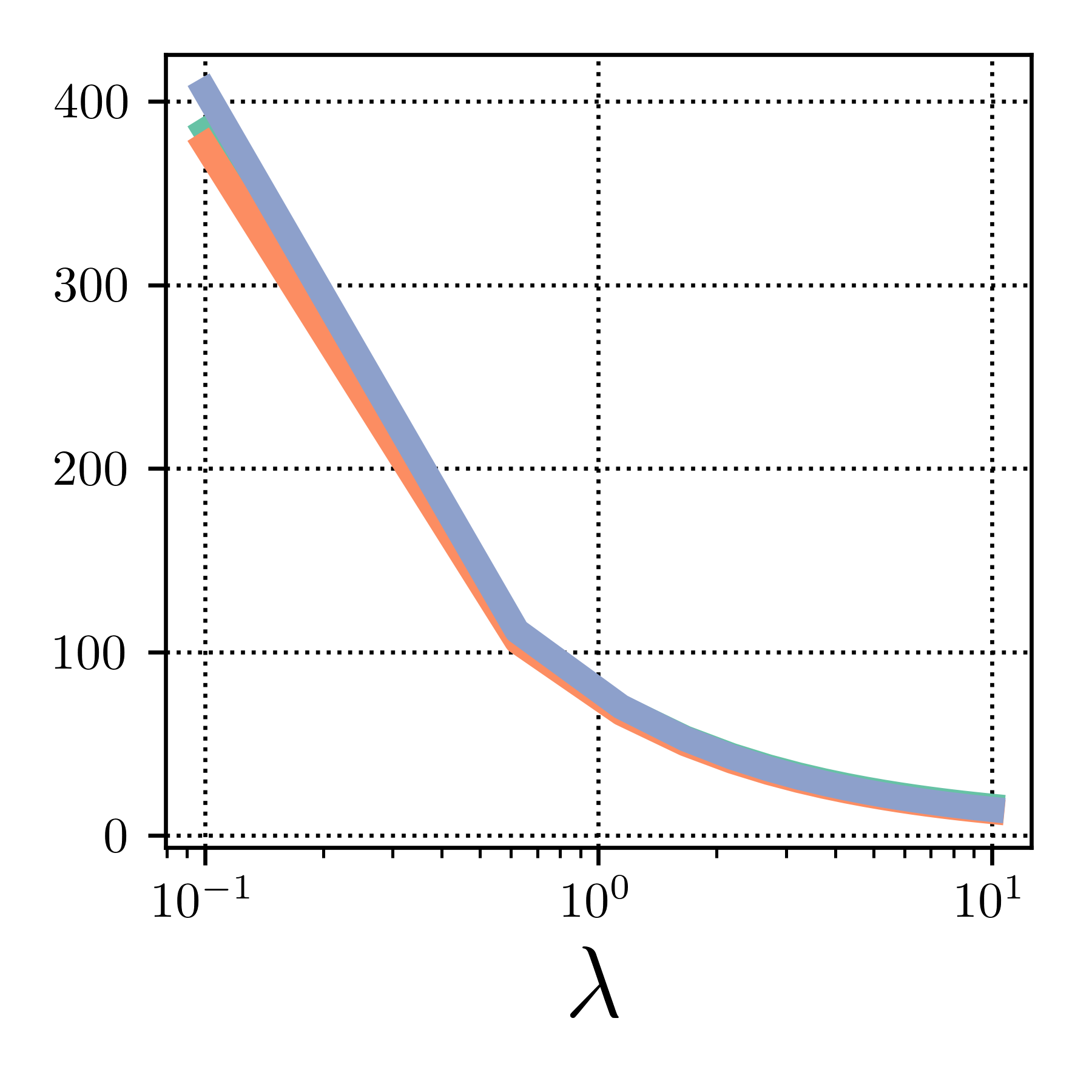}
\end{subfigure}%

\begin{subfigure}{.33\textwidth}
  \centering
  \includegraphics[width=\textwidth]{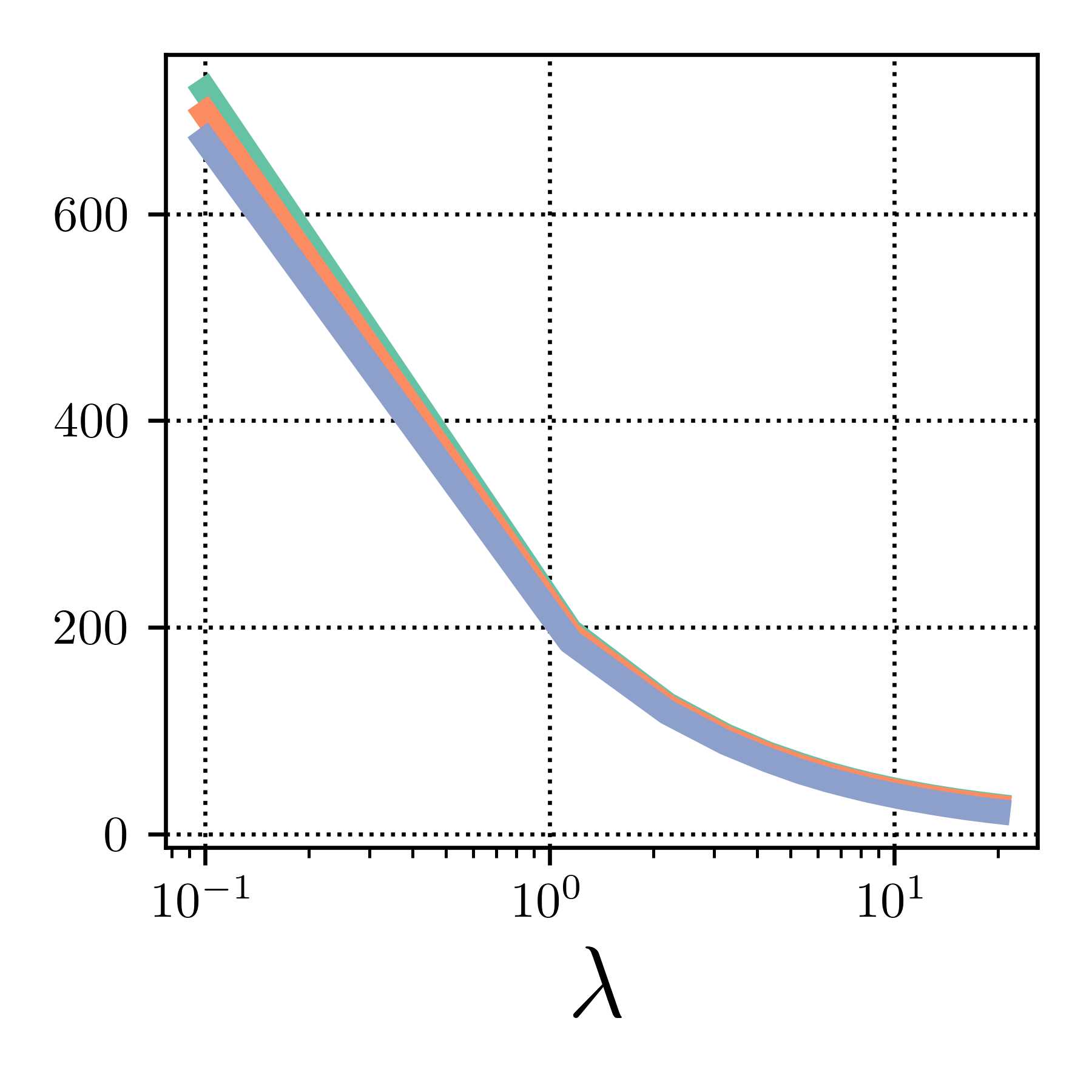}
\end{subfigure}%
\begin{subfigure}{.33\textwidth}
  \centering
  \includegraphics[width=\textwidth]{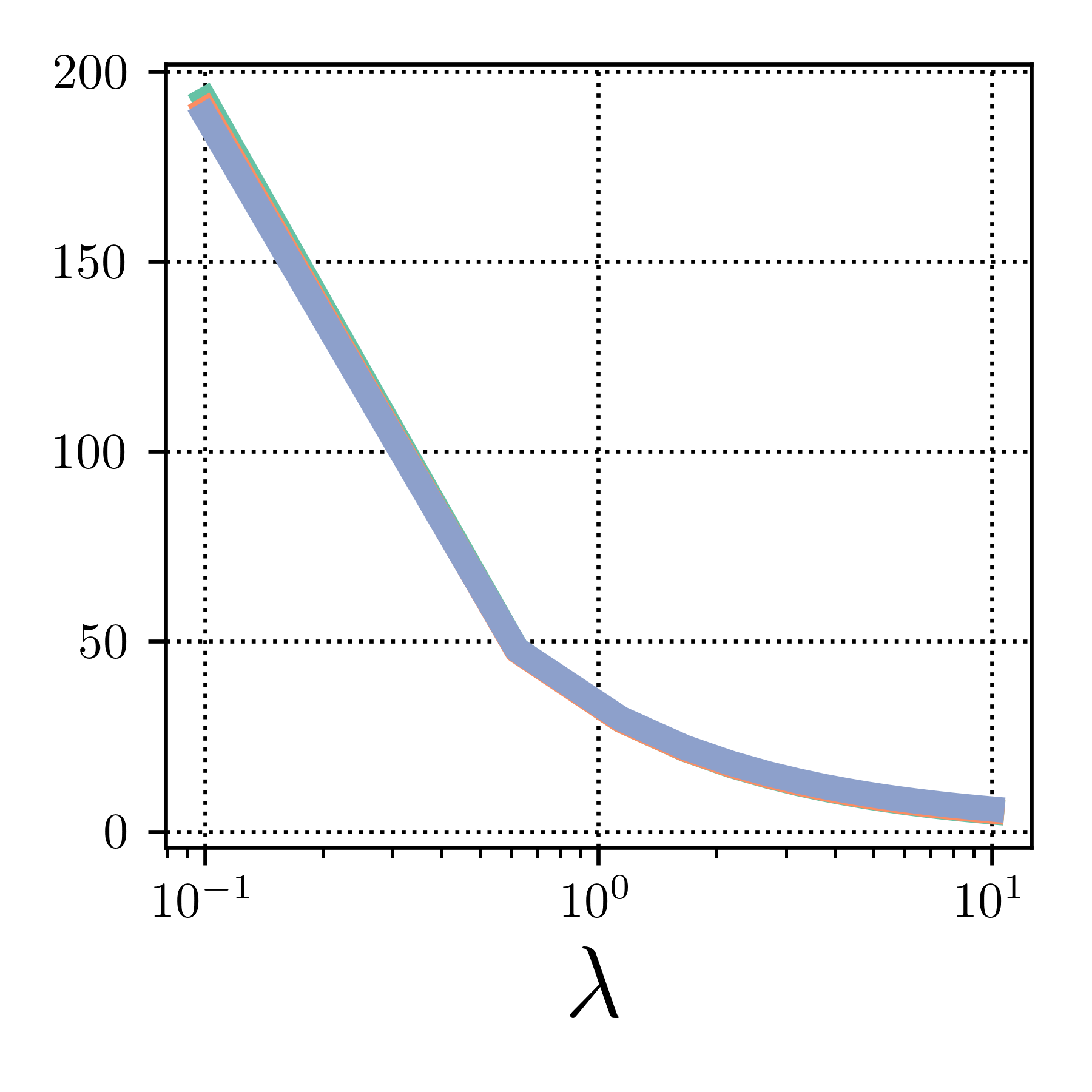}
\end{subfigure}%
\begin{subfigure}{.33\textwidth}
  \centering
  \includegraphics[width=\textwidth]{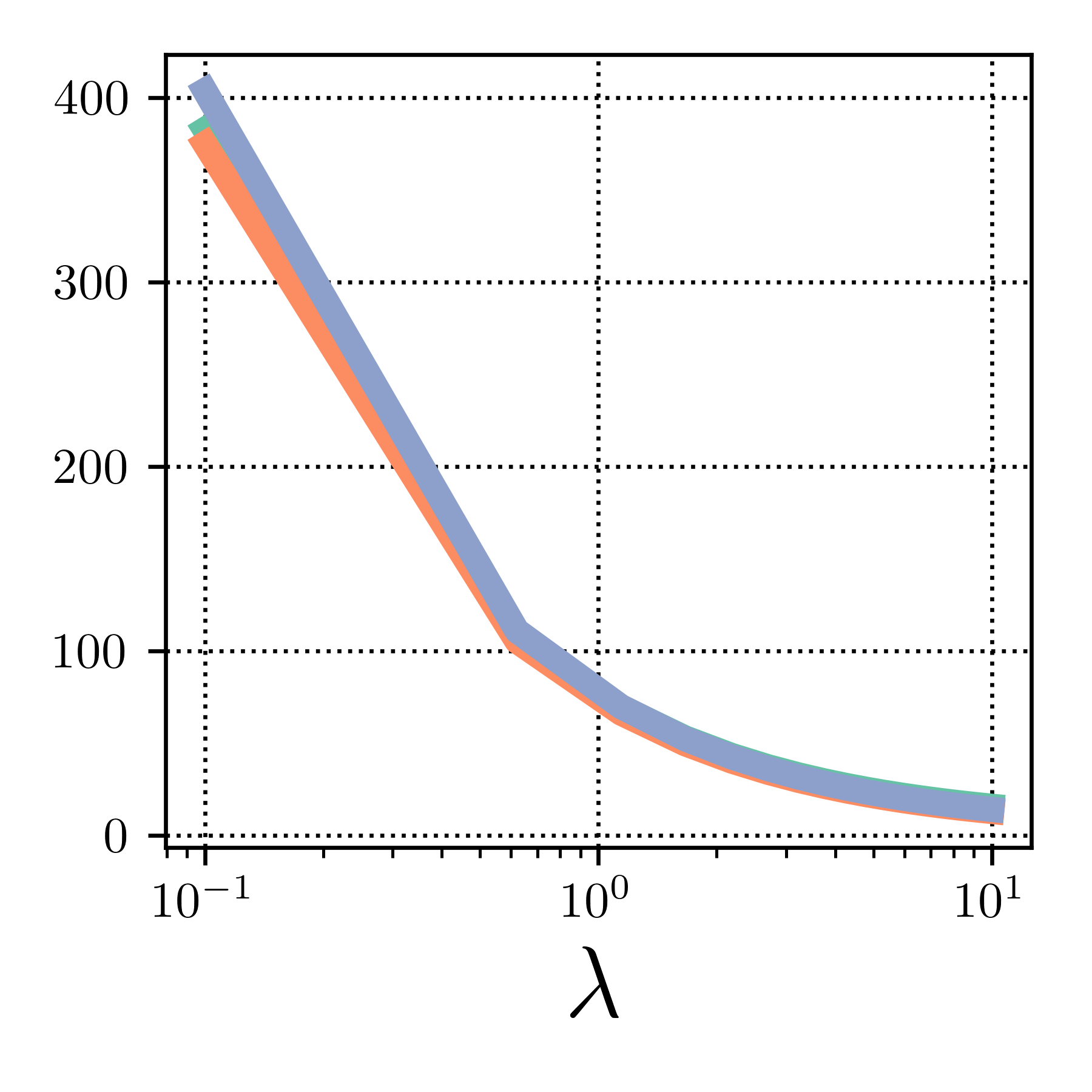}
\end{subfigure}%

\begin{subfigure}{.33\textwidth}
  \centering
  \includegraphics[width=\textwidth]{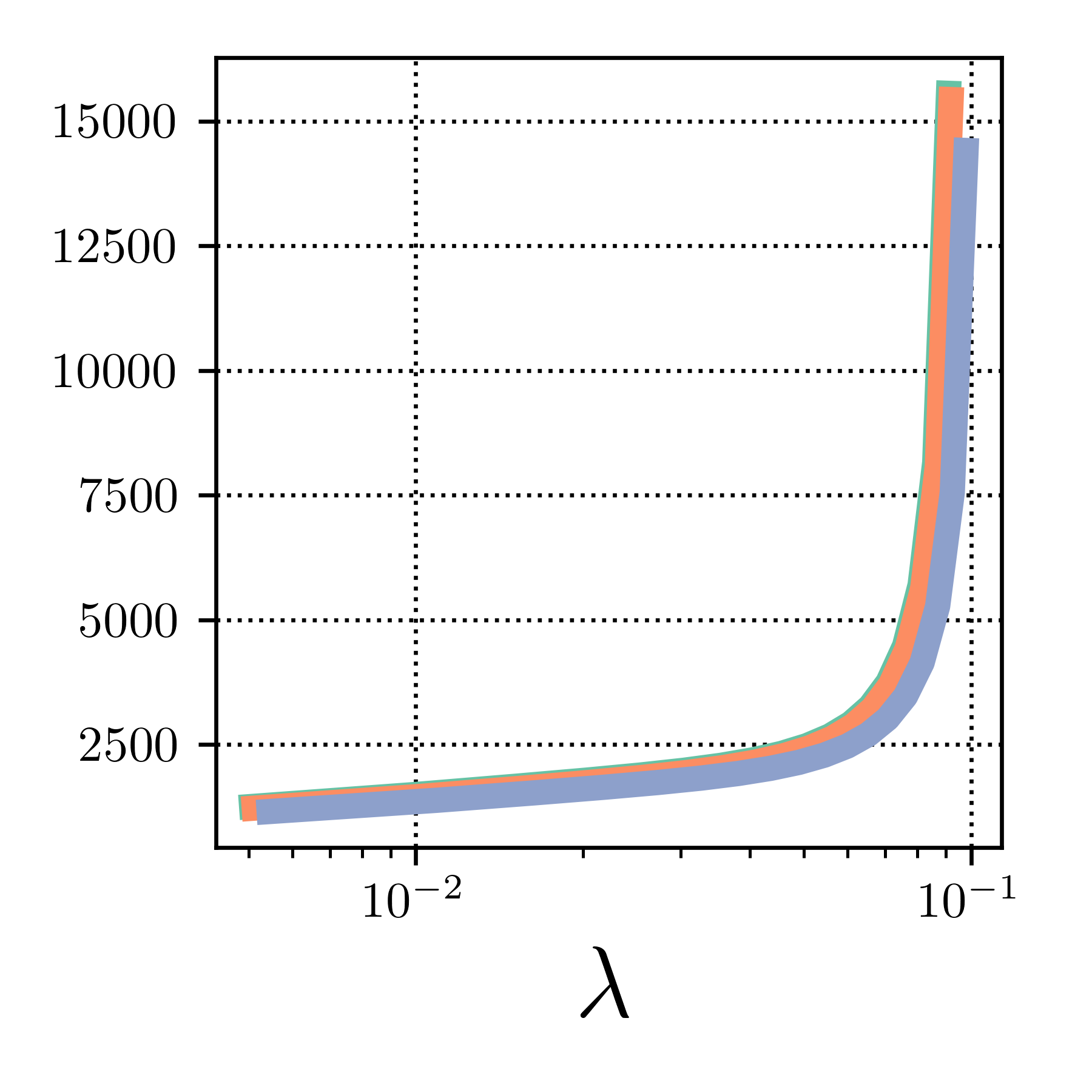}
  \caption{UCI, Abalone}
\end{subfigure}%
\begin{subfigure}{.33\textwidth}
  \centering
  \includegraphics[width=\textwidth]{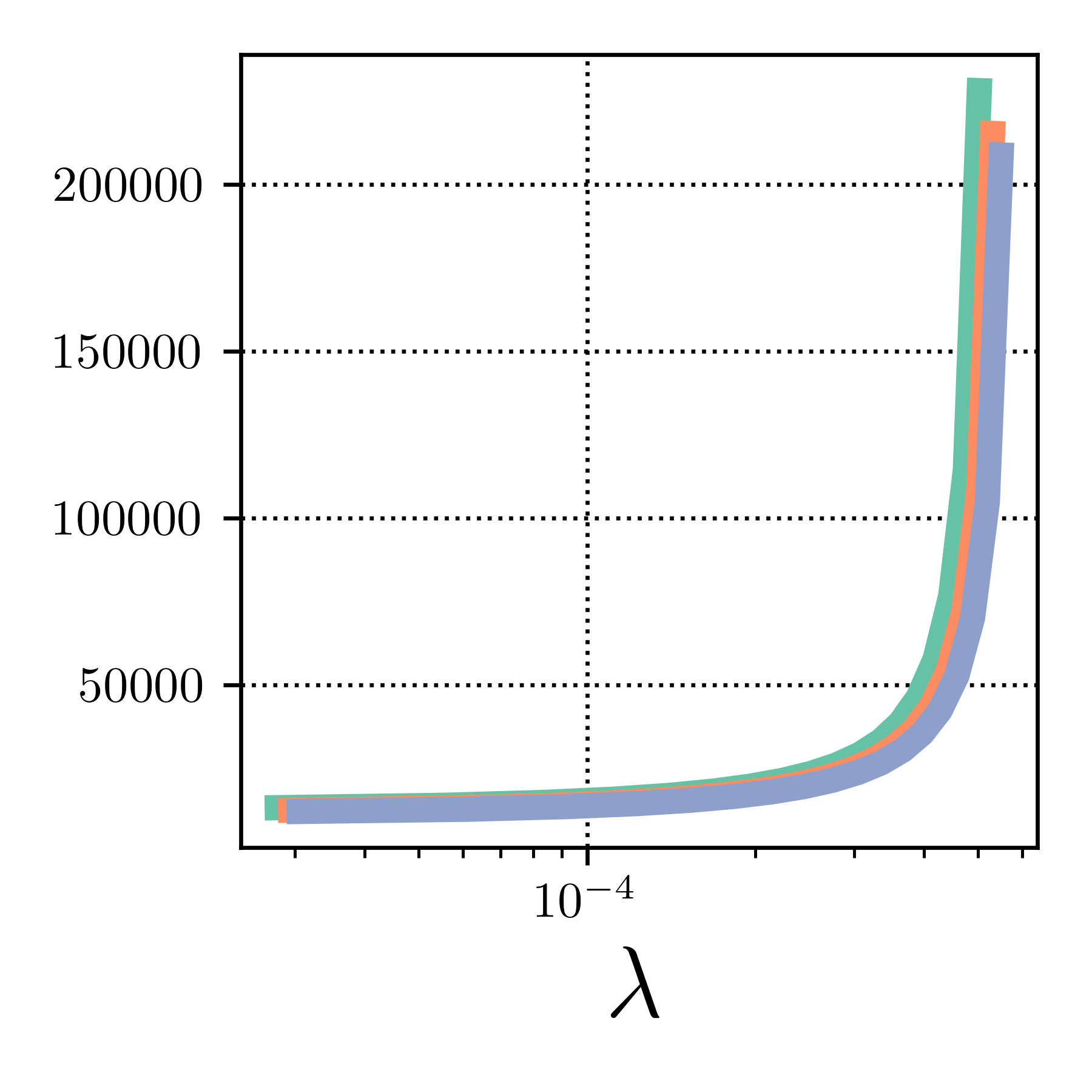}
  \caption{UCI, Diamonds}
\end{subfigure}%
\begin{subfigure}{.33\textwidth}
  \centering
  \includegraphics[width=\textwidth]{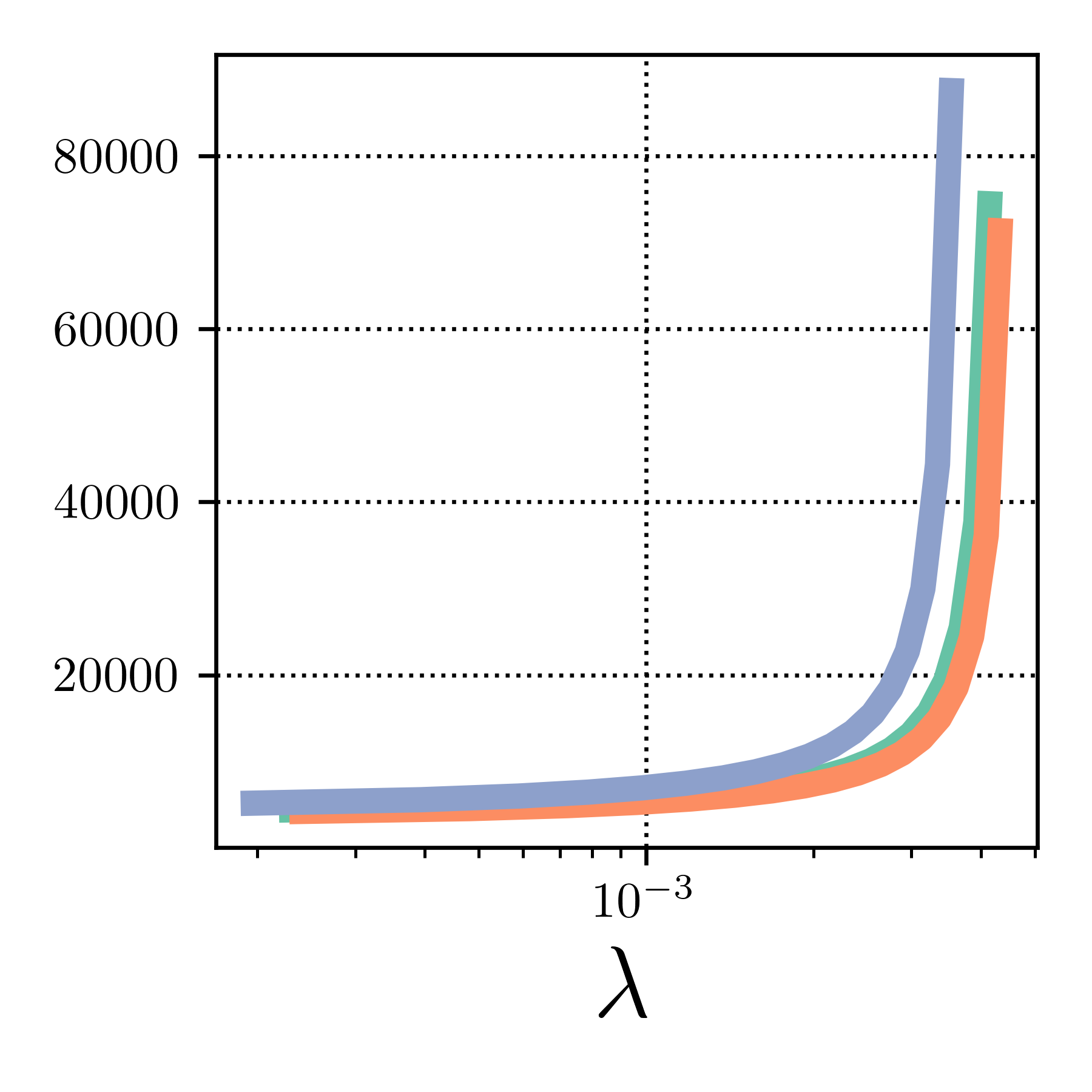}
  \caption{KC\_House}
\end{subfigure}%

  \caption{$\sigma^2_{\pi}=0.04737$. First row: Negative log-likelihood. Second row: PAC-Bayes original bound $\mathcal{B}_{\mathrm{original}}$. Third row: PAC-Bayes mixed bound $\mathcal{B}_{\mathrm{mixed}}$. Fourth row: PAC-Bayes approximate bound $\mathcal{B}_{\mathrm{approximate}}$ for varying $\lambda$. Different colours correspond to different MAP estimates. All quantities show very little variation around their mean, hence the significant overlap.}
  \label{exp_fig_detailed:figure_full}
\end{figure*}

\subsection{Additional classification results}
In Figure \ref{additional_metrics_mnist:figure_full} we see additional results for the MNIST-10 dataset. Specifically in addition to the Negative Log-Likelihood and the $\mathcal{B}_{\mathrm{original}}$ bound we also plot two other metrics, the Expected Calibration Error (ECE) \cite{naeini2015obtaining} and the Zero-One loss $\ell_{01}(f,\bx,y)=\mathbb{I}(\argmax (f(\bx))=y )$. We see that the ECE behaves similarly to the NLL and slowly improves with increasing $\lambda$. The misclassification rate (Zero-One loss) stays the same for all values of $\lambda$. In fact the noise we add to our parameters remains very small throughout all values of $\lambda$ and $\sigma_{\pi}^2$. Notice that in our setting the posterior precision is given by

$$\frac{1}{\sigma^2_{\hat{\rho}}}=\frac{\lambda h}{d}+\frac{1}{\sigma_{\pi}^2}.$$ 

In the above $\frac{h}{d}$ is empirically quite large $\frac{h}{d}\approx \mathcal{O}(10^4)$. Therefore even if we increase $\sigma_{\pi}^2$ significantly, we cannot increase $\sigma^2_{\hat{\rho}}$ to more than $\frac{1}{\frac{\lambda h}{d}}$. This explains why the NLL and the ECE change, while the Zero-One Loss remains constant. We are adding enough noise to change the probabilities of each class at the output of the neural network, but not enough noise to change the predicted class for each signal. Note also that our PAC-Bayes bounds including $\mathcal{B}_{\mathrm{original}}$ are an upper bound on the test NLL, and up to an unknown constant on $\mathrm{KL}(p_{\mathcal{D}}(y|\bx)\Vert \bE_{f\sim\hat{\rho}}[p(y|\bx,f)])$ (which the ECE can be seen to measure). Thus it is reasonable to observe the $\mathcal{B}_{\mathrm{original}}$ track the NLL and the ECE, and not the Zero-One Loss. Furthermore in other applications of the Laplace approximation such as in \citet{daxberger2021laplace,ritter2018scalable}, Laplace appoximations to the posterior typically have similar misclassification rate as their deterministic counterpart while the metrics that improve are the NLL and the ECE.

\begin{figure*}[h!]
\centering
\begin{subfigure}{.25\textwidth}
  \centering
  \includegraphics[width=\textwidth]{img/experiments/classification/MNIST/mnist_nll.png}
  \caption{NLL}
   \label{additional_metrics_mnist:figure1}
\end{subfigure}%
\begin{subfigure}{.25\textwidth}
  \centering
  \includegraphics[width=\textwidth]{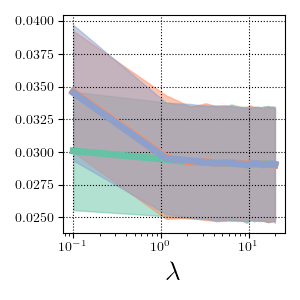}
  \caption{ECE}
   \label{additional_metrics_mnist:figure2}
\end{subfigure}%
\begin{subfigure}{.25\textwidth}
  \centering
  \includegraphics[width=\textwidth]{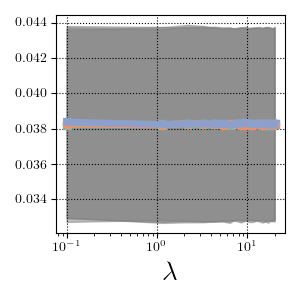}
  \caption{Zero-One Loss}
   \label{additional_metrics_mnist:figure3}
\end{subfigure}%
\begin{subfigure}{.25\textwidth}
  \centering
  \includegraphics[width=\textwidth]{img/experiments/classification/MNIST/mnist_original.png}
  \caption{$\mathcal{B}_{\mathrm{original}}$}
   \label{additional_metrics_mnist:figure4}
\end{subfigure}%

  \caption{The test NLL (\ref{additional_metrics_mnist:figure1}), ECE (\ref{additional_metrics_mnist:figure2}), Zero-One Loss (\ref{additional_metrics_mnist:figure3}), and the $\mathcal{B}_{\mathrm{original}}$  (\ref{additional_metrics_mnist:figure4}) bounds for different values of $\lambda$ and $\sigma_{\pi}^2$. We see that the ECE behaves similarly to the NLL and slowly improves with increasing $\lambda$. The misclassification rate (Zero-One loss) stays the same for all values of $\lambda$. In fact the noise we add to our parameters remains very small throughout all values of $\lambda$ and $\sigma_{\pi}^2$.} 
  \label{additional_metrics_mnist:figure_full}
\end{figure*}

\subsection{Additional results on bound terms}
We now present additional results on the behaviour of the different terms (Empirical Risk, KL, Moment) of the different bounds ($\mathcal{B}_{\mathrm{approximate}}$ , $\mathcal{B}_{\mathrm{mixed}}$, $\mathcal{B}_{\mathrm{original}}$). We plot the results in Figure \ref{bound_terms:figure_full}. We see that across all cases the $\mathcal{B}_{\mathrm{approximate}}$ bound is significantly off scale in the x-axis. For our realistic choices of $\sigma_{\pi}^2$ the parameter $\lambda$ is restricted to be $\lambda<1$, which is very limiting for the setting we want to investigate $\lambda\geq1$. However the bound gives some useful intuition regarding how the terms vary as we change $\lambda$. Across all datasets the Empirical Risk decreases as we increase $\lambda$, the KL term also increases, while the Moment term increases, but with a much slower rate than is implied by $\mathcal{B}_{\mathrm{approximate}}$. Comparing the $\mathcal{B}_{\mathrm{mixed}}$ and $\mathcal{B}_{\mathrm{original}}$ bounds we see that we get a significantly tighter estimate of the Empirical Risk. However this doesn't improve significantly the bound, in that the KL term is orders of magnitude larger than the other terms. Contrary to the main text we plot the original values for all quantities but on logarithmic scale when necessary. (In the main text we first normalize the KL and then add it to the other terms, which gives a result that is a little more difficult to interpret). The results for each dataset are for a random choice of $\sigma_{\pi}^2$.

\begin{figure*}[t!]
\centering
\begin{subfigure}{\textwidth}
\centering
\begin{subfigure}{.33\textwidth}
  \centering
  \includegraphics[width=\textwidth]{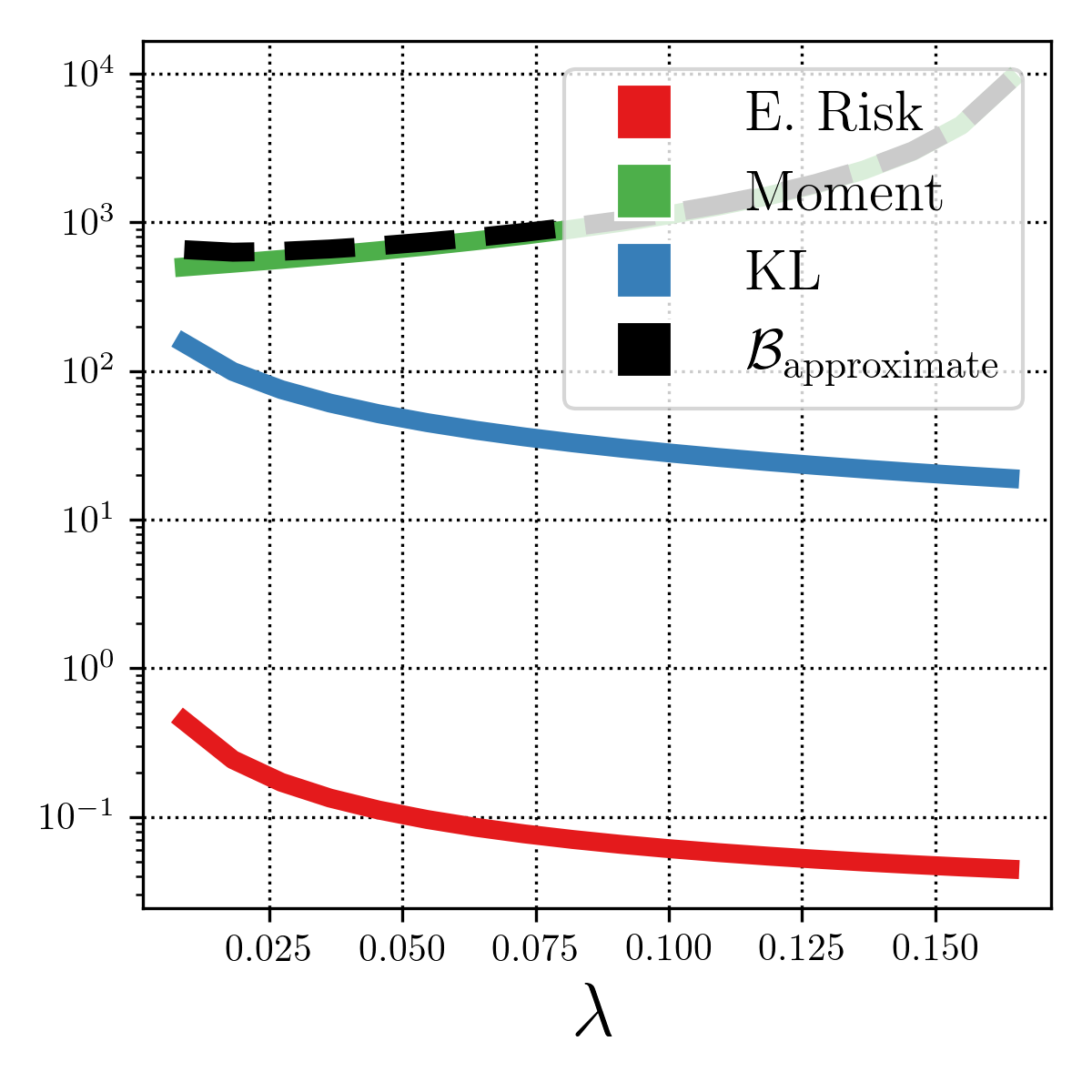}
\end{subfigure}%
\begin{subfigure}{.33\textwidth}
  \centering
  \includegraphics[width=\textwidth]{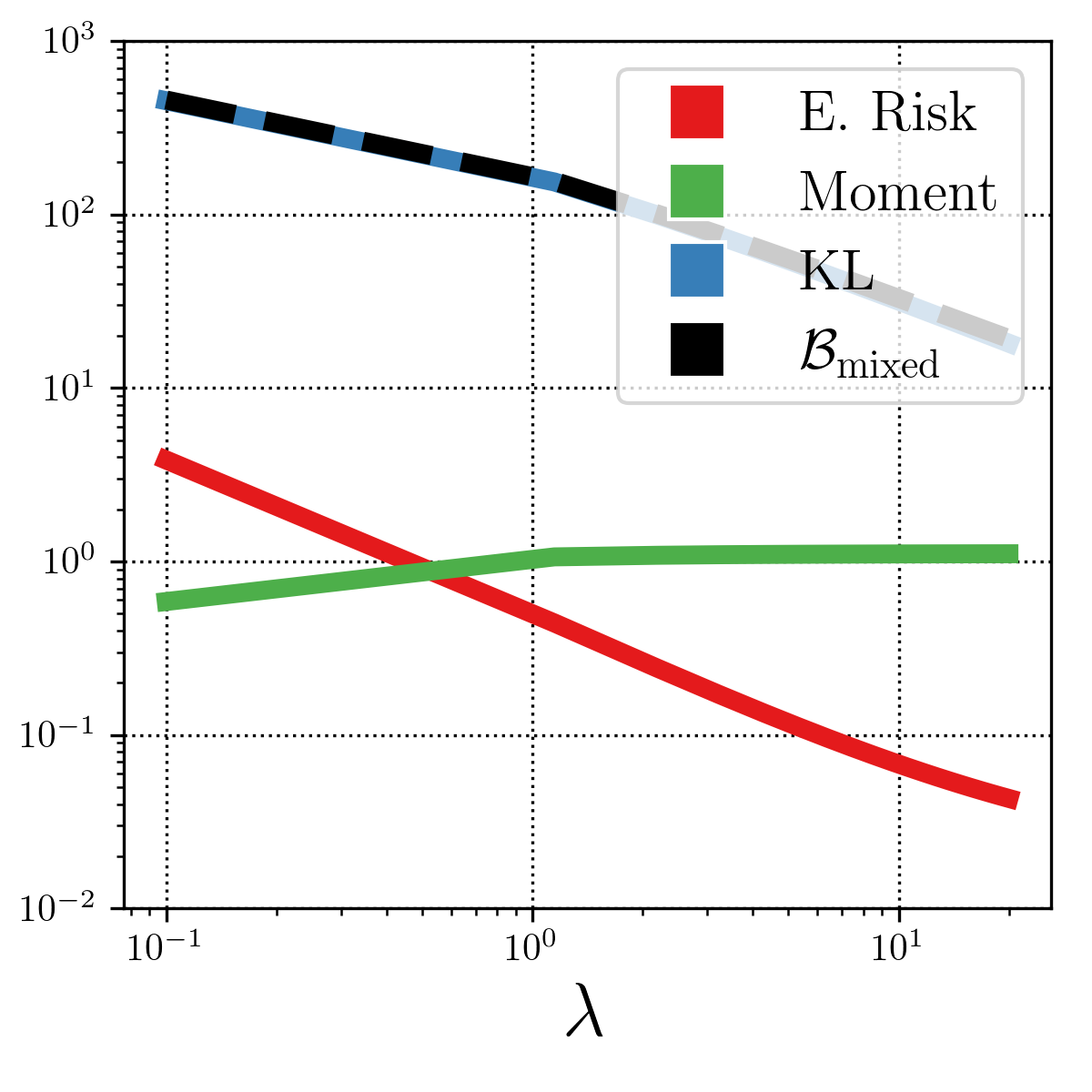}
\end{subfigure}%
\begin{subfigure}{.33\textwidth}
  \centering
  \includegraphics[width=\textwidth]{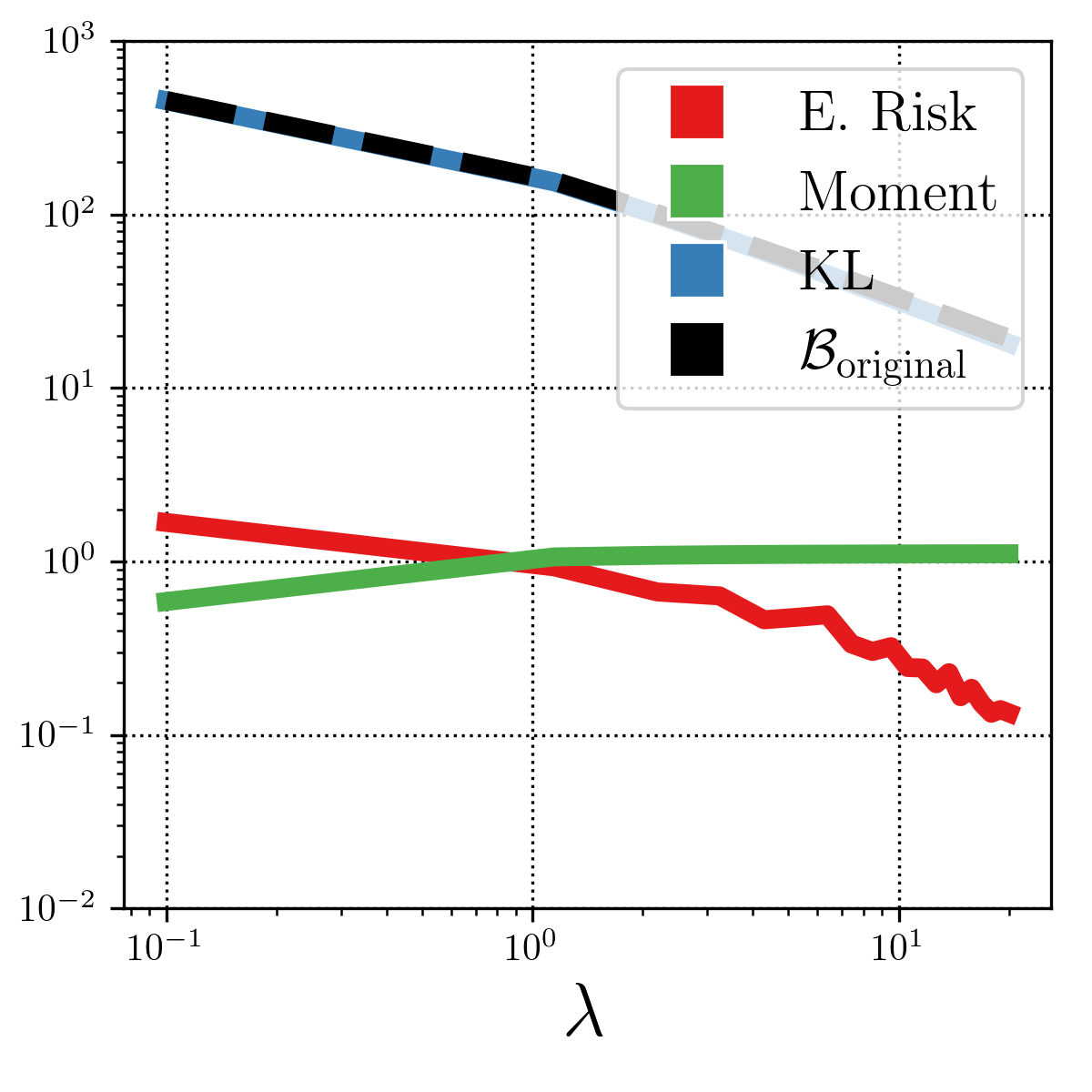}
\end{subfigure}%
\caption{Abalone}
\end{subfigure}

\begin{subfigure}{\textwidth}
\centering
\begin{subfigure}{.33\textwidth}
  \centering
  \includegraphics[width=\textwidth]{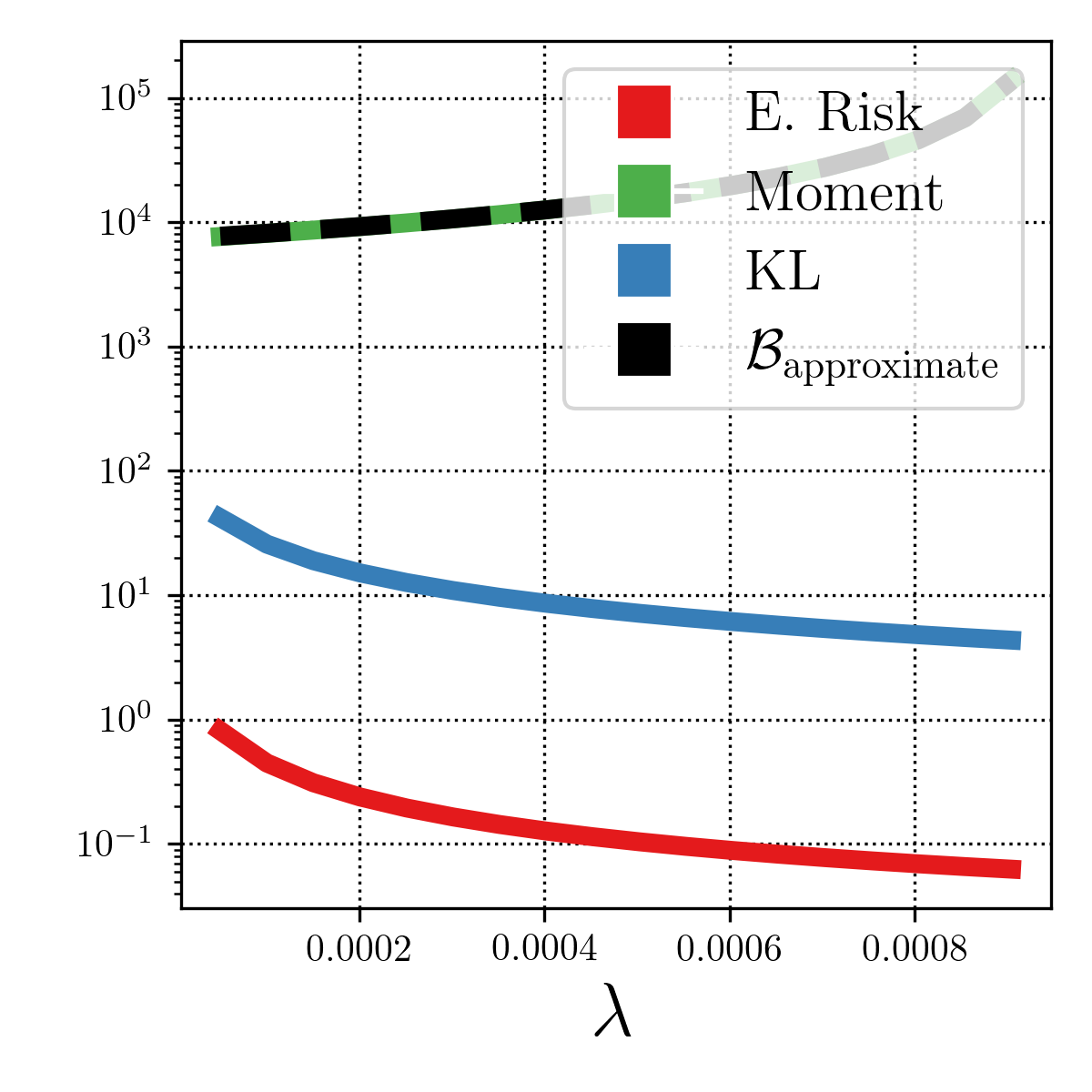}
   \label{bound_terms:figure1}
\end{subfigure}%
\begin{subfigure}{.33\textwidth}
  \centering
  \includegraphics[width=\textwidth]{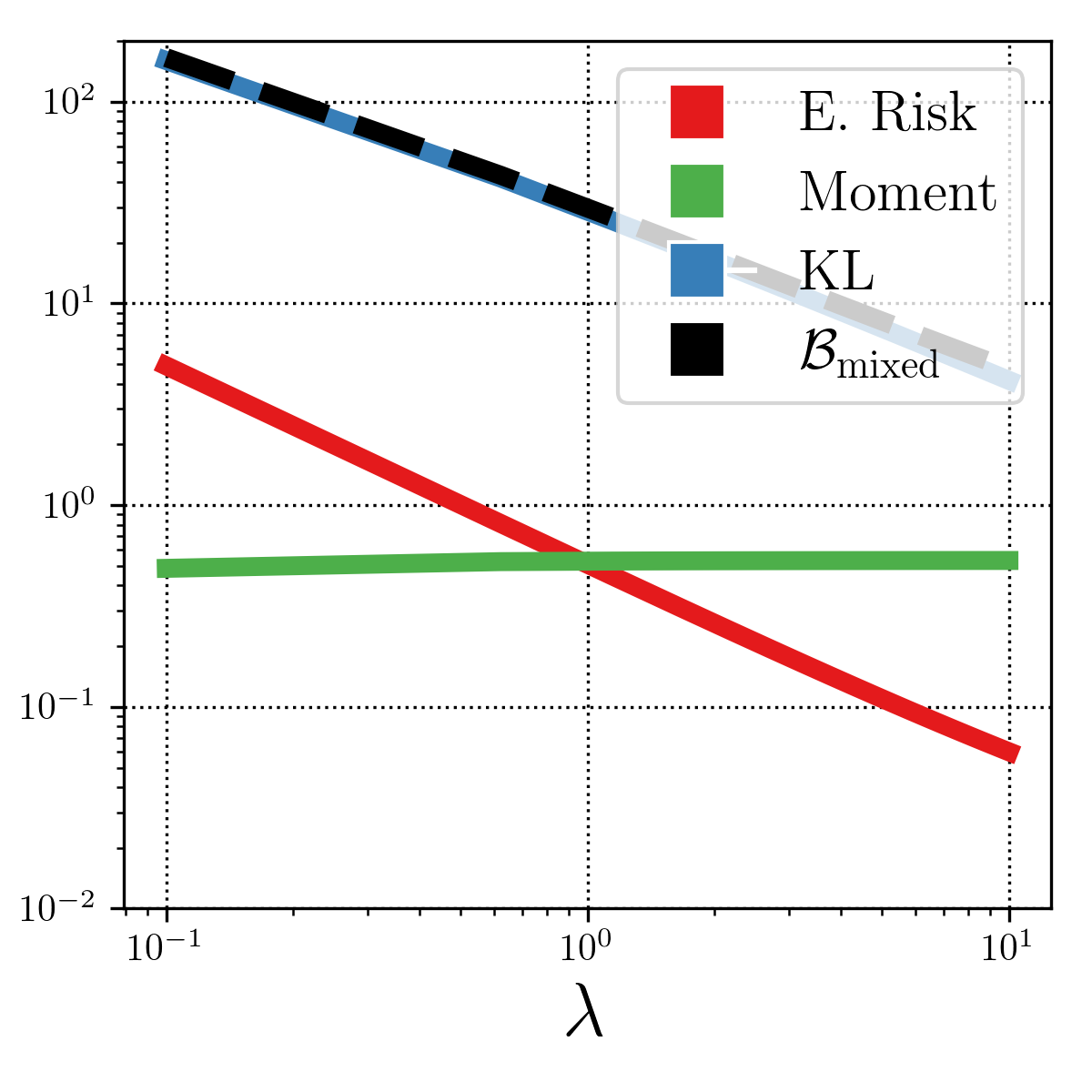}
   \label{bound_terms:figure2}
\end{subfigure}%
\begin{subfigure}{.33\textwidth}
  \centering
  \includegraphics[width=\textwidth]{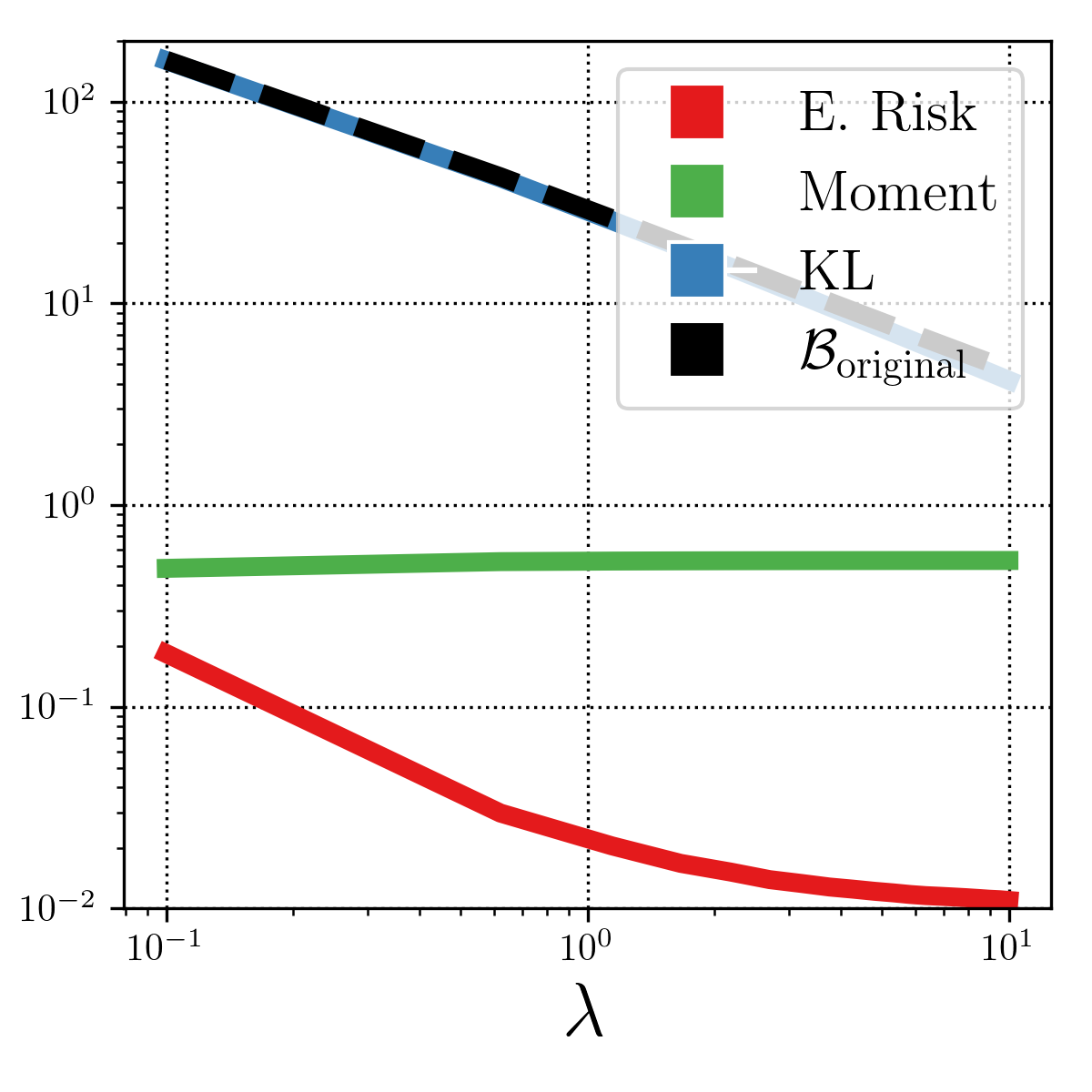}
   \label{bound_terms:figure3}
\end{subfigure}%
\caption{Diamonds}
\end{subfigure}

\begin{subfigure}{\textwidth}
\centering
\begin{subfigure}{.33\textwidth}
  \centering
  \includegraphics[width=\textwidth]{img/theory/kc_house/kc_house_approximate.png}
   \label{bound_terms:figure4}
\end{subfigure}%
\begin{subfigure}{.33\textwidth}
  \centering
  \includegraphics[width=\textwidth]{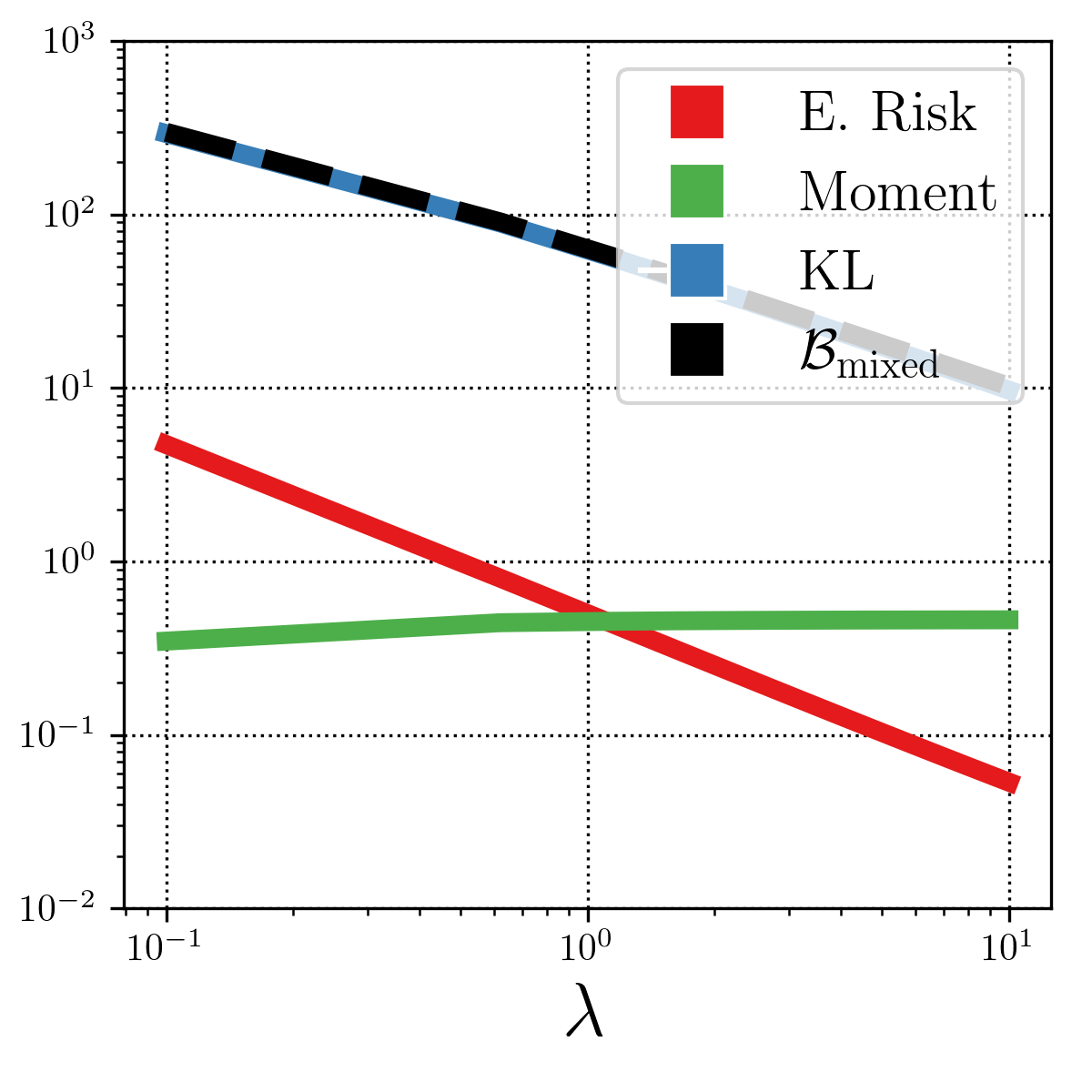}
   \label{bound_terms:figure5}
\end{subfigure}%
\begin{subfigure}{.33\textwidth}
  \centering
  \includegraphics[width=\textwidth]{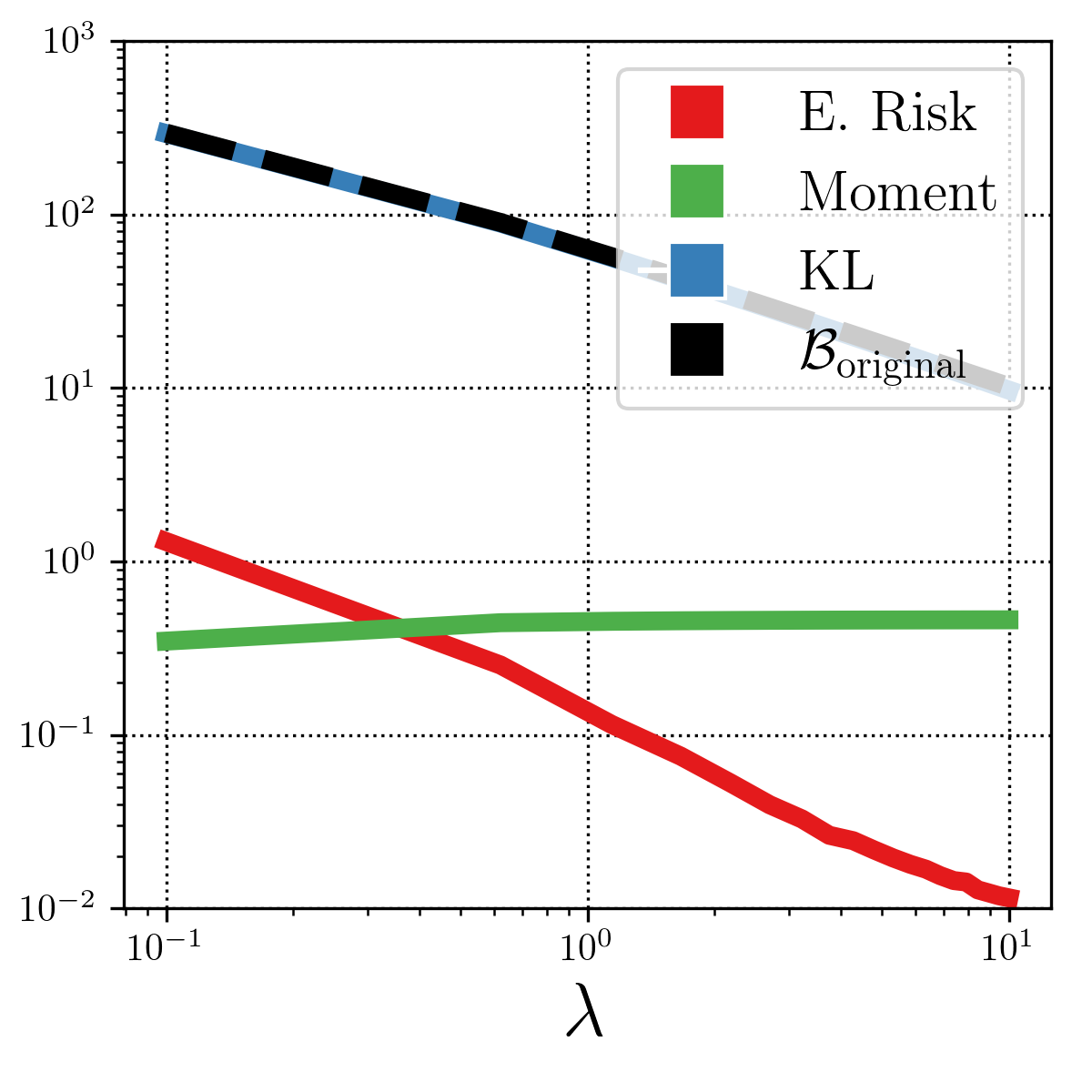}
   \label{bound_terms:figure6}
\end{subfigure}%
\caption{KC\_House}
\end{subfigure}

\begin{subfigure}{\textwidth}
\centering
\hfill 
\begin{subfigure}{.33\textwidth}
  \centering
  \includegraphics[width=\textwidth]{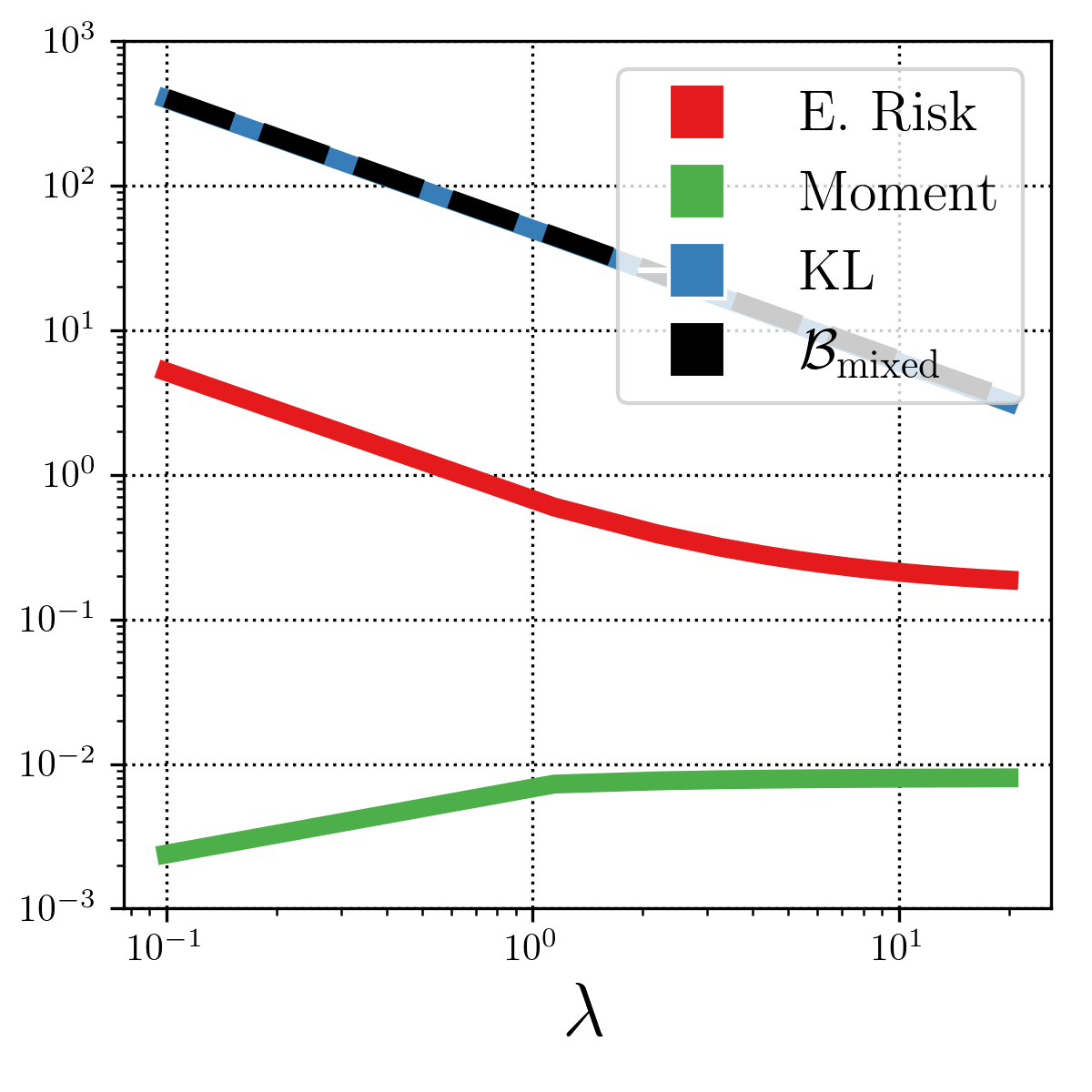}
   \label{bound_terms:figure8}
\end{subfigure}%
\begin{subfigure}{.33\textwidth}
  \centering
  \includegraphics[width=\textwidth]{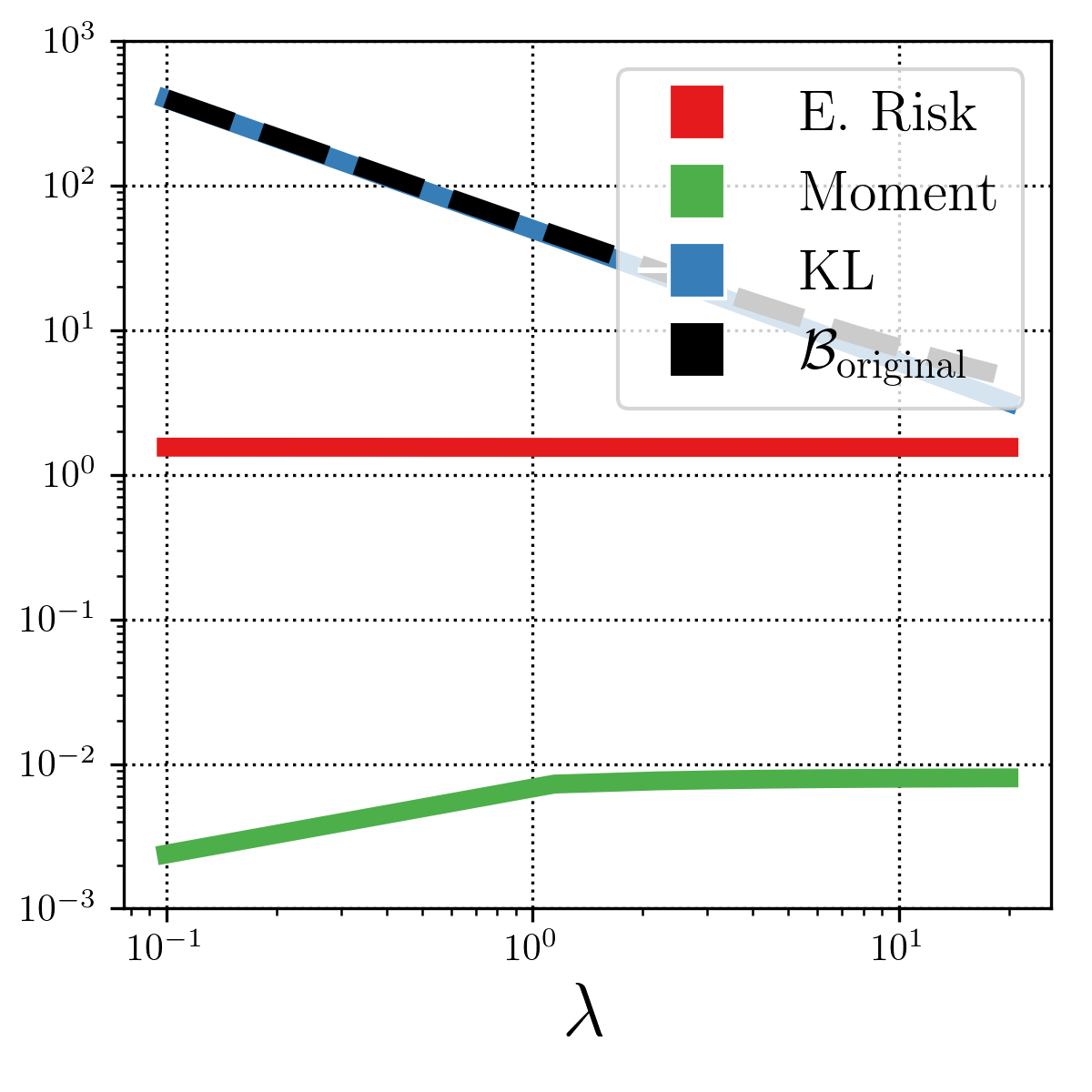}
   \label{bound_terms:figure9}
\end{subfigure}%
\caption{MNIST-10}
\end{subfigure}

  \caption{The different terms of the $\mathcal{B}_{\mathrm{approximate}}$ (left column), $\mathcal{B}_{\mathrm{mixed}}$ (middle column), and $\mathcal{B}_{\mathrm{original}}$ (right column) bounds. We see that the $\mathcal{B}_{\mathrm{approximate}}$ bound gives some useful intuition. Moving from $\mathcal{B}_{\mathrm{mixed}}$ to $\mathcal{B}_{\mathrm{original}}$ gives some improvements to the Empirical Risk values but not to the bound values which are orders of magnitude larger. }
  \label{bound_terms:figure_full}
\end{figure*}

\clearpage

\section{Proof of Theorem 1}
We include here a proof of Theorem 1, first presented in \citet{germain2016pac}, and based on \cite{alquier2016properties,begin2016pac,germain2016pac} to illustrate how the Moment term and the temperature parameter $\lambda$ arise in the final bound.
The Donsker--Varadhan's change of measure states that, for any measurable function $\phi:\mathcal{F}\rightarrow\mathbb{R}$, we have
\begin{equation*}
    \bE_{f\sim\hat{\rho}}\phi(f)\leq \mathrm{KL}(\hat{\rho}\Vert\pi)+\ln(\bE_{f\sim\pi}\exp[\phi(f)]).
\end{equation*}
Thus, with $\phi(f):=\lambda\left(\mathcal{L}_{\mathcal{D}}^{\ell}(f)-\hat{\mathcal{L}}_{X,Y}^{\ell}(f) \right)$, we obtain $\forall \hat{\rho}$ on $\mathcal{F}$:
\begin{equation*}
\begin{split}
\lambda\left(\bE_{f\sim\hat{\rho}}\mathcal{L}_{\mathcal{D}}^{\ell}(f)-\bE_{f\sim\hat{\rho}}\hat{\mathcal{L}}_{X,Y}^{\ell}(f) \right)&=\bE_{f\sim\hat{\rho}}\lambda\left(\mathcal{L}_{\mathcal{D}}^{\ell}(f)-\hat{\mathcal{L}}_{X,Y}^{\ell}(f) \right)\\
&\leq \mathrm{KL}(\hat{\rho}\Vert\pi)+\ln\left(\bE_{f\sim\pi}\exp[\lambda\left(\mathcal{L}_{\mathcal{D}}^{\ell}(f)-\hat{\mathcal{L}}_{X,Y}^{\ell}(f) \right)]\right).
\end{split}
\end{equation*}
Now, we apply Markov's inequality on the random variable $\zeta_{\pi}(X,Y):=\bE_{f\sim\pi}\exp\left[\lambda\left(\mathcal{L}_{\mathcal{D}}^{\ell}(f)-\hat{\mathcal{L}}_{X,Y}^{\ell}(f) \right)\right]$ and get
\begin{equation*}
    \mathrm{Pr}_{(X,Y)\sim\mathcal{D}^n}\left(\zeta_{\pi}(X,Y) \leq\frac{1}{\delta}\bE_{(X',Y')\sim\mathcal{D}^n}\zeta_{\pi}(X',Y') \right)\geq1-\delta.
\end{equation*}
This implies that with probability at least $1-\delta$ over the choice $(X,Y)\sim\mathcal{D}^n$, we have $\forall\hat{\rho}$ on $\mathcal{F}$
\begin{equation*}
\bE_{f\sim\hat{\rho}}\mathcal{L}_{\mathcal{D}}^{\ell}(f)\leq\bE_{f\sim\hat{\rho}}\hat{\mathcal{L}}_{X,Y}^{\ell}(f)+\frac{1}{\lambda}\left[ \mathrm{KL}(\hat{\rho}\Vert\pi)+\ln \frac{\bE_{(X',Y')\sim\mathcal{D}^n}\bE_{f\sim\pi}\exp[\lambda\left(\mathcal{L}_{\mathcal{D}}^{\ell}(f)-\hat{\mathcal{L}}_{X,Y}^{\ell}(f) \right)]}{\delta} \right].
\end{equation*}

\end{document}